\documentclass{article}

\usepackage{microtype}
\usepackage{graphicx}
\usepackage{subfigure}
\usepackage{booktabs} 

\usepackage{hyperref}

\usepackage[accepted]{icml2021}

\usepackage{url}
\usepackage{amsmath,amsfonts,amsthm,amssymb,bm, verbatim,dsfont,mathtools}
\usepackage{algorithm,algorithmic,color,graphicx,appendix}
\usepackage{epsfig,epsf,bbm,mathabx}
\usepackage{cases}
\usepackage{subfigure}
\usepackage{float}
\usepackage{etoolbox}
\usepackage{xr,xspace}
\usepackage{todonotes}
\usepackage{paralist}
\usepackage{caption}
\usepackage{multirow}

\theoremstyle{plain}
\newtheorem{theorem}{Theorem}
\newtheorem{lemma}{Lemma}

\theoremstyle{definition}
\newtheorem{definition}{Definition}

\newtheorem*{remark*}{Remark}

\usepackage{xin_macros}

\begin{document}

\twocolumn[
\icmltitle{A Probabilistic Approach to Neural Network Pruning}



\icmlsetsymbol{equal}{*}

\begin{icmlauthorlist}
\icmlauthor{Xin Qian}{to}
\icmlauthor{Diego Klabjan}{to}
\end{icmlauthorlist}

\icmlaffiliation{to}{Department of Industrial Engineering and Management Science, Northwestern University}

\icmlcorrespondingauthor{Xin Qian}{xinqian2017@u.northwestern.edu}
\icmlcorrespondingauthor{Diego Klabjan}{d-klabjan@northwestern.edu}

\icmlkeywords{Machine Learning, ICML}

\vskip 0.3in
]



\printAffiliationsAndNotice{}  

\begin{abstract}
Neural network pruning techniques reduce the number of parameters without compromising predicting ability of a network. Many algorithms have been developed for pruning both over-parameterized fully-connected networks (FCNs) and convolutional neural networks (CNNs), but analytical studies of capabilities and compression ratios of such pruned sub-networks are lacking. We theoretically study the performance of two pruning techniques (random and magnitude-based) on FCNs and CNNs. Given a target network {whose weights are independently sampled from appropriate distributions}, we provide a universal approach to bound the gap between a pruned and the target network in a probabilistic sense. The results establish that there exist pruned networks with expressive power within any specified bound from the target network.
\end{abstract}

\section{Introduction}

The common neural network architectures that achieve the state-of-the-art results usually have tens of billions of trainable parameters \citep{goodfellow2016deep, radford2019language, brown2020language}, leading to a problem that training and inference of these models are computationally expensive and memory intensive. To address this problem, researchers have developed many practical algorithms to compress the network structure while keeping the original network's expressive power \citep{li2016pruning, han2015deep, han2015learning, cheng2017survey}. 

Recently, \citet{frankle2018lottery} conjecture that, every successfully trained neural network contains much smaller subnetworks (winning tickets) that---when trained in isolation from the original initialization---reach test accuracy comparable to the original network. This conjecture is called the Lottery Ticket Hypothesis (LTH). \citet{ramanujan2020s} further conjecture that, a sufficiently over-parameterized neural network with random initialization contains subnetworks that can achieve competitive accuracy without any training, when comparing to a large trained network. This conjecture can be viewed as a stronger version of the LTH in the sense that we do not need to train this over-parameterized random network. However, to determine the lottery tickets from this over-parameterized network is NP-Hard in the worst case \citep{malach2020proving, pensia2020optimal}. In addition, since over-parameterization is compared to a trained neural network, which is usually already over-parameterized, the random initialized network is over-over-parameterized and thus too large to consider. 

Although the development of such network pruning algorithms dates back to late 80s, there have been only limited studies of the theoretical guarantees of network pruning. The existence and the representation power of good subnetworks are lacking (see the related sections of the survey papers \citep{sun2019optimization, fan2019selective}). Recently, \citet{malach2020proving} prove the strong LTH for fully-connected networks with ReLU activations. They show that, given a target FCN of depth $l$ and width $d$, any random initialized network with depth $2l$ and width $O\pth{d^5l^2/\epsilon^2}$ contains subnetworks that can approximate the target network with $\epsilon$ error with high probability. In the following works, \citet{pensia2020optimal} and \citet{orseau2020logarithmic} concurrently and independently prove that the width of the random initialized network can be reduced to $O\pth{d\log(dl/\epsilon)}$. \citet{pensia2020optimal} further show that this logarithmic over-parameterization is essentially optimal for networks with constant depth. Note that a random initialized network is introduced and pruning is applied on this new network, instead of the target network. Thus, these results cannot provide much insights for developing model pruning algorithms that are applied on the target network directly. Besides, the proof ideas heavily rely on the fact the the random initialized network is well over-parameterized so that a subnetwork with a specific structure that can replicate a single neuron of the target network exists. Although the researchers have improved the polynomial dependency of width to logarithm, the size of the random initialized network is still very large. 

In this work, we focus on the theoretical results of pruning an over-parameterized target network directly. There are two types of subnetworks: subnetworks where specific weights are pruned (weight-subnetwork) and subnetworks where entire neurons are pruned (neuron-subnetworks). We focus on weight-subnetworks and show that, for both magnitude-based pruning (prune the smallest entries in the weight matrices based on magnitude) and random pruning (randomly select some entries in the weight matrices to prune), we can prune some weights of the target network while maintaining comparable expressive power with positive probability. We show the {random pruning and magnitude-based pruning results for FCNs and the random pruning result for CNNs}, where the latter one requires a sophisticated formulation to translate the convolutional layers into fully-connected layers with a specific ciuculant structure. The proof framework, which bounds the gap between the output of the pruned and trained networks layer by layer, are universal for both FCNs and CNNs. We rely on the results from probability theory and theoretical computer sciences to give precise bounds of the norms of weight matrices and other random variables. 

Our results, {as one of the rare studies about the existence of good subnetworks,} provide relationships between the width of the target network, the {number of pruned weights in} each layer, the universal expressive error over a closed region, and the probability that such good subnetworks exist. These results also give guidance for practical researchers by providing the probability that a good subnetwork exists and an estimation of how many entries can be pruned at one time by magnitude-based pruning and random pruning. 

The rest of the manuscript is structured as follows. In Section \ref{sec:lit_review} we review the literature while in Section \ref{sec:notations} we show the preliminaries and notations that are used throughout the paper. Sections \ref{sec:FCN} and \ref{sec:CNN} discuss the theoretical results of pruning of FCNs and CNNs, respectively. We conclude the results and discuss the potential future works in Section \ref{sec:conclusion}. We present technical lemmas in Appendix \ref{appendix:sec:lemma} and the complete proofs of the theorems in Appendix \ref{appendix:sec:proof}. In Appendix \ref{appendix:sec:magnitude} we discuss how to extend the theorems to more general settings and in Appendix \ref{appendix:sec:numerical} we show some numerical results that support our theorems and assumptions.

\section{Literature Review} \label{sec:lit_review}

\paragraph{Empirical Neural Network Pruning} There has been a long history of neural network pruning. Early studies of pruning reduce the number of connections based on the information of second-order derivatives of the loss function \citep{lecun1989optimal}. Following works focus on magnitude-based pruning. \citet{han2015deep} propose to reduce the total number of parameters and operations in the entire network. Other works explore pruning neurons and design various methods to determine the redundancy of neurons \citep{hu2016network, Srinivas2015DatafreePP}. Similar approaches are also applied to CNNs to prune filters \citep{luo2017thinet} or entire convolutional channels \citep{li2016pruning}. Recently, \citet{frankle2018lottery} conjecture the lottery ticket hypothesis that, a trained network contains a subnetwork that---when trained in isolation from the original initialization---can match the performance of the original network. \citet{NEURIPS2019_1113d7a7} claim that the good subnetworks in the LTH have better-than-random performance without any training. Based on the above two works, \citet{ramanujan2020s} conjecture the so-called strong LTH that, within a sufficiently over-parameterized neural network (comparing to the target network) with random weights at initialization, there exists a subnetwork that achieves competitive accuracy with the target network.

\paragraph{Theoretical Study of Neural Network Pruning} The study of the theoretical properties of neural network pruning only started recently. \citet{malach2020proving} prove the strong LTH for FCNs with ReLU activations. In particular, they show that one can approximate any target FCN of width $d$ and depth $l$ by pruning a sufficiently over-parameterized network of width $O(d^5 l^2/\epsilon^2)$ and depth $2l$ such that the gap between the pruned and target networks is bounded by $\epsilon$. \citet{pensia2020optimal} and \citet{orseau2020logarithmic} concurrently and independently improved the width of the random network to $O(\textrm{poly}(d) \log(dl/\epsilon))$. These results are based on the idea that, for a single-neuron ReLU connection, we can use a two-hidden-layer neural network with constant width to approximate it. In comparison, our results study pruning of the target FCNs and CNNs directly. Another line of research by \citet{ye2020good, ye2020greedy} propose a greedy optimization based neural network pruning method. They also provide theoretical guarantees of the decreasing discrepancy between the pruned and target networks. \citet{elesedy2020lottery} stick with the iterative magnitude-based pruning procedure described in \citet{frankle2018lottery} and prove the LTH for linear models trained by gradient flow methods. {\citet{pmlr-v80-arora18b} and \citet{zhou2019nonvacuous} theoretically study 
a close connection between compressibility and generalization of neural networks. Another line of work \citep{baykal2018datadependent, Liebenwein2020Provable, baykal2019sipping} propose sampling-based neural network pruning algorithms according to certain sensitivity scores and provide theoretical guarantees for both FCNs and CNNs.}

\paragraph{Theoretical Study of CNNs}
Although CNNs are successful in many computer vision tasks \citep{goodfellow2016deep}, there is less work discussing theoretical properties of CNNs. \citet{jain1989} shows that a linear transformation of a 2D convolutional filter can be represented by a doubly block circulant matrix. The circulant structure provides an efficient way to calculate the singular values of the linear transformation corresponding to a convoultional layer \citep{sedghi2018singular}.

\section{Preliminaries and Notations} \label{sec:notations}

We introduce some notations that are used in the sequel. For vector $v$, we use $\norm{v}_0$ and $\norm{v}_2$ to denote the $L_0$ and $L_2$ norm of $v$, respectively. For matrix $M \in \reals^{m \times n}$, we use $M_{i,j}$ or $\pth{M}_{i,j}$ to denote the element in the $i$-th row and $j$-th column of $M$; we use $M_{i. :}$ and $M_{:, j}$ to denote the $i$-th row and $j$-th column of $M$, respectively; the vectorization of $M$ is defined as $\textrm{vec}(M) \defeq \qth{M_{1,1}, \ldots, M_{m,1}, \ldots, M_{1,n}, \ldots, M_{m,n}}^T$. We also use analogous notations for higher-order tensors. The operator norm and element-wise maximum norm of $M$ is denoted by $\norm{M}_2$ and $\norm{M}_{\max} \defeq \max_{i\in[m], j\in[n]} \abs{A_{i,j}}$, respectively. The Hadamard (element-wise) product of two matrices $A, B \in \reals^{m\times n}$ is denoted by $M \defeq A\circ B$, where $M_{i,j} = A_{i,j}B_{i,j}$. We denote $\vec{0}_{m\times n}$ and $\vec{1}_{m\times n}$ as the zero matrix and all 1 matrix of dimension $m \times n$.

For $n \in \mathbb{N}^+$, we define $[n] \defeq \sth{1,\ldots, n}$. Given a sequence of real numbers ${a_n}$, we denote $\prod_{k=i}^j a_k = a_i \times \cdots \times a_j$ if $j \ge i$ and $\prod_{k=i}^j a_k = 1$ otherwise. For integers $n>0$ and $k$, we use $k\%n = k \textrm{ mod } n$ if $n \nmid k$ and $k\%n = n$ otherwise\footnote{Note that this definition is slightly different from the common definition of modulo.}. We use $\log$ to denote the natural logarithm and set $\log^{(2)}(x) \defeq \log(\log(x))$.

We use $\calU\qth{a, b}$ to denote the uniform distribution on interval $[a, b]$, and $\calN\pth{\mu, \Sigma}$ to denote the multivariate normal distribution with mean $\mu \in \reals^p$ and covariance matrix $\Sigma \in \reals^{p\times p}$. 

We are given a target neural network $F$ of depth $l \ge 3$ of the form \begin{equation} \label{eq:target_network_general_form}
        F(x) = W_l^* \sigma_l \pth{W_{l-1}^* \sigma_{l-1} \pth{\cdots W_2^* \sigma_1 \pth{W_1^* x}}}
    \end{equation}
where $\sigma_k$ is the activation function and weight matrix $ W_k^* \in \reals^{d_{k} \times d_{k-1}}, k\in [l]$\footnote{Throughout the paper, we skip the bias terms in the expression of the neural network.}. Typically, there are two types of subnetworks, namely weight-subnetworks and neuron-subnetworks, depending on whether we remove (or set to zero) the entire neuron or just the entries of a weight matrix. In this paper, we focus the theoretical results on weight-subnetworks. Mathematically, a pruned weight-subnetwork $f$ of $F$ is a network of the same architecture as $F$ such that the weight matrix in the $k$-th layer of $f$ is represented by $W_k = M_k \circ W_k^*$ for some mask $M_k \in \sth{0, 1}^{d_k \times d_{k-1}}$. Throughout the paper, we fix $M_1$ and $M_l$ as the all 1 matrix (i.e. we do not prune any weight on the first and last weight matrix of the target network). {We aim at reducing the number of active weights while keeping the expressive power of the original network $F$.}

The compression ratio of the $k$-th layer is defined as $\gamma_k \defeq  \norm{\textrm{vec}\pth{W_k}}_{0} /D_k$, where $D_k \defeq d_k d_{k-1}$ is the number of weights in the $k$-th layer. Obviously, we aim at reducing the compression ratios while keeping the expressive power of the original network $F$. 

The error metric used throughout the paper is the universal approximation over the unit ball $\calB_{d_0} \defeq \sth{x\in\reals^{d_0}: \norm{x}_2 \le 1}$, or in the CNN results we use the unit cube $\calC_{d_0} \defeq \sth{x \in \reals^{d_0}: x_i \in [0, 1], i \in [d_0]}$ instead; i.e. $f$ is $\epsilon$-close to $F$ if \begin{equation*}
    \sup_{x \in \calB_{d_0}} \norm{f(x) - F(x)}_2 \le \epsilon.
\end{equation*}
{This definition of discrepancy is common in the theoretical model pruning literature \citep{malach2020proving, pensia2020optimal, orseau2020logarithmic}.} Note that the results of this paper can be easily generalized from the unit ball to any ball with radius $r$ in $\reals^{d_0}$. We use the unit ball (or unit cube) only for ease of notation. {The discrepancy between the losses of the pruned and target network on a given set of samples can be derived similarly.}


\section{Pruning Fully-connected Neural Networks} \label{sec:FCN}

In this section, we show that a fully-connected neural network can be approximated by pruning some of its entries while keeping comparable expressive power under mild assumptions.

We start with two different pruning approaches -- random pruning and magnitude-based pruning. Given a target network $F$ as defined in \eqref{eq:target_network_general_form} and compression ratios $\gamma_k, k \in [l]$, random pruning refers to applying a set of masks $\sth{M_1, \ldots, M_l}$ on $F$ such that $M_k$ is constructed by starting with $M_k = \vec{1}_{d_k d_{k-1}}$ and repeating $\lfloor \gamma_k D_k \rfloor$ times the following steps: (1) select $i \in [d_{k}]$ uniformly at random; (2) select $j \in [d_{k-1}]$ uniformly at random; (3) set $\pth{M_k}_{i,j} = 0$\footnote{{Note that this scheme corresponds to ``with-replacement'' sampling, i.e., an index pair $(i,j)$ might be selected twice. There is another ``without-replacement'' strategy. For more details regarding these two strategies, please refer to Appendix \ref{appendix:subsec:diff}.}}. The magnitude-based pruning refers to applying a set of masks $\sth{M_1, \ldots, M_l}$ on $F$ such that $\pth{M_k}_{i,j} = {0}$ if $(i,j) \in \calI_k$ and $\pth{M_k}_{i,j} = {1}$ otherwise, where we order the entries of $W_k^*$ such that $\abs{W_k^*}_{i_1, j_1} \le \cdots \le \abs{W_k^*}_{i_{D_k}, j_{D_k}}$ and set $\calI_k \defeq \sth{(i_u,j_u): 1 \le u \le \lfloor \gamma_k D_k \rfloor}$. Recall that we assume $\gamma_1 = \gamma_l = 1$ and thus $M_1 $ and $M_l$ are all 1 matrices\footnote{{There is another global version of magnitude-based pruning where the weights of the entire network are sorted and the weights with the smallest magnitudes are pruned. For comparison between these two approaches, please refer to Appendix \ref{appendix:subsec:magnitude}.}}. 

Our main theorems in this section show that, for both pruning approaches and under mild conditions, the target network $F$ contains a weight-subnetwork that is $\epsilon$-close to $F$ with high probability. We present the results for magnitude-based pruning and random pruning in Sections \ref{subsec:magnitude_FCN} and \ref{subsec:random_FCN}, respectively. We outline the proof in Section \ref{subsec:scratch_FCN} and defer the complete proof to Appendix \ref{appendix:sec:proof}.

\subsection{Magnitude-based Pruning of FCN} \label{subsec:magnitude_FCN}

We first present the result for magnitude-based pruning. 

\begin{theorem} \label{thm:FCN_uniform}
    We are given a target network $F$ as defined in \eqref{eq:target_network_general_form}. Let us assume that \begin{enumerate}[(i)]
        \vspace{-0.5em}\item $\sigma_k$ is $L_k$-Lipschitz and satisfies $\sigma_k(0) = 0, k \in [l]$;
        \vspace{-0.5em}\item $d \defeq \min\sth{d_1, \ldots, d_{l-1}} \ge \max\sth{d_0, d_l}$;
        \vspace{-0.5em}\item entries in $W_k^*$ are independent and identically distributed following $\calU\qth{-\frac{K}{\sqrt{\max\sth{d_{k}, d_{k-1}}}}, \frac{K}{\sqrt{\max\sth{d_{k}, d_{k-1}}}}}$ for a fixed positive constant $K$.
    \end{enumerate} \vspace{-0.5em}Let $\epsilon > 0, \delta > 0$, and $\alpha \in (0, 1)$ be such that \begin{equation*}
        d \ge \max\sth{C_1^{\frac{1}{\alpha}}, \pth{\frac{C_2}{\epsilon}}^{\frac{1}{\alpha}}, \pth{\frac{C_3}{\delta}}^{\frac{1}{\alpha}}, C_4 + C_5 \log\pth{\frac{1}{\delta}}}
    \end{equation*} for some positive constants $C_1, C_2, C_3, C_4$ and $C_5$ (depending on $l$ and $L_k$'s) as specified in the proof. Then with probability at least $1-\delta$, the subnetwork $f$ of $F$ with mask $M = \sth{M_1, \ldots, M_l, M_k \in \sth{0, 1}^{d_{k}\times d_{k-1}}}$ that prunes the smallest $\lfloor D_k^{1-\alpha} \rfloor$ entries of $W_k^*, 1<k<l$ based on magnitude is $\epsilon$-close to $F$, i.e. \begin{equation} \label{eq:prune_target_approximation_FCN_uniform}
        \sup_{x\in \calB_{d_0}} \norm{f(x) - F(x)}_2 \le \epsilon.
    \end{equation}
\end{theorem}

Note that many activation functions, like ReLU and tanh, hold for assumption (i) with $L_k=1$. In assumption (ii), we assume that the width of the target neural network is larger than the input and output dimensions. This is common in most of the theoretical and practical deep learning results. For assumption (iii), we take the upper/lower bound of the uniform distribution to be $\pm \frac{K}{\sqrt{\max\sth{d_{k}, d_{k-1}}}}$ for a fixed positive constant $K$ so that the variance of this distribution is of the same order as in the Xavier initialization \citep{glorot2010understanding}. We are aware of the fact that for many trained networks, the weights in each layer do not fit a uniform distribution well. We use the uniform distribution since the closed-form of the order statistics is only available for this distribution. We utilize these closed-form results to give a precise relationship between the width $d$, error $\epsilon$, probability $1-\delta$, and the compression ratio that depends on $\alpha$. Asymptotic results exist for order statistics of general distributions and can be used to estimate such relationships. We discuss more details on how to apply the results of intermediate order statistics to generalize Theorem \ref{thm:FCN_uniform} to other distributions in Appendix \ref{appendix:magnitude:FCN}. {The weights are assumed to be independent for simplicity. For the near-independent and non-independent settings, please refer to Appendix \ref{appendix:subsec:independent}. Same discussions about independency apply for Theorems \ref{thm:FCN_general} and \ref{thm:CNN}.}


\subsection{Random Pruning of FCN} \label{subsec:random_FCN}

In this section, we present the result for random pruning of FCNs. The key difference between random pruning and magnitude-based pruning is that, given the target network $F$, the mask corresponding to magnitude-based pruning is fixed while the mask of random pruning is random. 

Given the compression ratio $\gamma_k$ (or the number of weights to prune) in the $k$-th layer, a random pruning mask $M_k$ can be viewed as random selecting $\lfloor \gamma_k D_k \rfloor$ entries of $\sth{0,1}^{d_{k} \times d_{k-1}}$ with replacement and setting them to zero. These random selected masks are combined to form the mask $\sth{M_1, \ldots, M_l}$ that represents a random pruned weight-subnetwork of $F$. This random property further complicates the proof, as we need to consider the randomness from the entries of the target network and the randomness from the mask at the same time.

Besides the difference of the two pruning approaches, we only assume that each entry of the weight matrix independently follows a distribution with bounded second-order and fourth-order moments, while in Theorem \ref{thm:FCN_uniform} we assume that all the entries in the weight matrix are independently and identically following a specific distribution. 

\begin{theorem} \label{thm:FCN_general}
    We are given a target network $F$ as defined in \eqref{eq:target_network_general_form}. Let us assume that \begin{enumerate}[(i)]
        \vspace{-0.5em}\item $\sigma_k$ is $L_k$-Lipschitz and satisfies $\sigma_k(0) = 0, k \in [l]$;
        \vspace{-0.5em}\item $d \defeq \min\sth{d_1, \ldots, d_{l-1}} \ge \max\sth{d_0, d_l}$;
        \vspace{-0.5em}\item $\pth{W_k^*}_{i,j}$ independently follows a distribution $\calX^k_{i, j}$; further, there exist two positive constants $K_1$ and $K_2$ such that $\mathbb{E} {\calX_{i,j}^k}=0$, $\mathbb{E} \abs{\calX_{i,j}^k}^2 \le \frac{K_1}{\max\sth{d_{k}, d_{k-1}}}$ and $\mathbb{E} \abs{\calX_{i,j}^k}^4 \le \frac{K_2}{\max\sth{d_{k}, d_{k-1}}^2}$;
        \vspace{-0.5em}\item for all $k \in [l]$, there exists a positive constant $N_k$ such that $\norm{W_k^*}_2 \le N_k$ with probability at least $1-\delta_k$.
    \end{enumerate} \vspace{-0.5em} Let $\epsilon > 0, \delta > 0$, and $\alpha \in (0, 1)$ be such that {\small \begin{align} 
        \alpha &\le 1 - \frac{\log\pth{d_{k+1}+1} - \log^{(2)}\pth{d_{k+1}}}{\log\pth{d_{k+1}} + \log\pth{d_{k}}}, 1<k<l, \label{eq:assumption_balls_in_bins_for_alpha_1}\\
        \alpha &\le 1 - \frac{\log\pth{d_{k}+1} - \log^{(2)}\pth{d_{k}}}{\log\pth{d_{k+1}} + \log\pth{d_{k}}}, 1<k<l, \label{eq:assumption_balls_in_bins_for_alpha_2}\\
        \delta_0 & \defeq \delta - \qth{\delta_l + \sum_{i=1}^{l-1} (l-i)\delta_i} \ge 0, \label{eq:assumption_general_delta}
    \end{align}}
    and {\small \begin{equation*}
	    d \ge \max \sth{C_1^{\frac{4}{\alpha}}, \pth{\frac{C_2}{\epsilon}}^{\frac{4}{\alpha}}, \pth{\frac{C_3}{\delta_0}}^{3}, \pth{\frac{C_4}{\delta_0}}^{\frac{4}{\alpha}}},
	\end{equation*}}for some positive constants $C_1, C_2, C_3$ and $C_4$ (depending on $l$, $L_k$'s, and $N_k$'s) specified in the proof. Then with probability at least $1-\delta \ge \pth{1-d^{-\frac{1}{3}}}^{2(l-2)}\pth{1 - \delta_l} \Big[ 1 - (l-2) c_2 d^{-\frac{\alpha}{4}} - \sum_{i=1}^{l-1} (l-i)\delta_i\Big]$ over the randomness of masks and weights {for some positive constant $c_2$ defined in the proof}, the subnetwork $f$ of $F$ with mask $M = \sth{M_1, \ldots, M_l, M_k \in \sth{0, 1}^{d_{k}\times d_{k-1}}}$ that randomly prunes $\lfloor D_k^{1-\alpha} \rfloor$ entries of $W_k^*, 1< k < l$ is $\epsilon$-close to $F$, i.e., \begin{equation} \label{eq:prune_target_approximation_FCN_general}
        \sup_{x\in \calB_{d_0}} \norm{f(x) - F(x)}_2 \le \epsilon.
    \end{equation}
\end{theorem}

We next discuss the feasibility of these assumptions. Assumptions (i) and (ii) have already been used in Theorem \ref{thm:FCN_uniform}. These two assumptions are common in both practice and theory. Since the target network $F$ is usually a trained one, a universal distribution for all entries in a layer might not be realistic. Thus we have assumption (iii) to allow non-homogeneous distributions of the entries in the weight matrices. The two bounds $\mathbb{E} \abs{\calX_{i,j}^k}^2 \le \frac{K_1}{\max\sth{d_{k+1}, d_k}}$ and $\mathbb{E} \abs{\calX_{i,j}^k}^4 \le \frac{K_2}{\max\sth{d_{k+1}, d_k}^2}$ hold for a variety of distributions, like the uniform distribution, normal distribution, and sub-Gaussian distribution, as long as the variance of the distribution is set to $O\pth{\frac{1}{d}}$. This holds because, if we initialize the target network $F$ following the Xavier initialization and train the network properly, the variance of the weights should remain of the same order, approximately. We further verify that this assumption holds by checking the distribution of some trained FCNs and CNNs. We train a 5-hidden-layer FCN with 1024 neurons in each hidden layer on the \texttt{Covertype} dataset \citep{blackard1998comparative} by randomly selecting initial weights. Figure \ref{fig:hist} shows the histogram of weights in different layers of the trained FCN. They exhibit a sub-Gaussian distribution and the second-order and fourth-order moments are well bounded by $O\pth{\frac{1}{d}}$ and $O\pth{\frac{1}{d^2}}$, respectively. See Appendix \ref{appendix:subsec:distribution} for more details. Assumption (iv) bounds the operator norm of the weight matrices, which is an important term in the proof. Without loss of generality, we assume that $N_k \ge 1$. We can also have $\delta \ge \delta_l + \sum_{i=1}^{l-1} (l-i)\delta_i$. This can be achieved by increasing the value of $N_k$ and thus reducing the value of $\delta_k$. However, we should carefully choose the values of $N_k$'s and $\delta_k$'s, as larger $N_k$'s also increase the lower-bound of the minimum number of neurons in the target network. In fact, assumption (iv) with certain $N_k$'s and $\delta_k$'s can be derived from assumption (iii) with Lemma \ref{lemma:latala} and the Markov's inequality. We use assumption (iv) as it allows possible tighter values.

\begin{figure*}[h!]
    \centering
    \begin{tabular}{cccccc}
        \hspace{-1.5em}\includegraphics[width=0.35\textwidth]{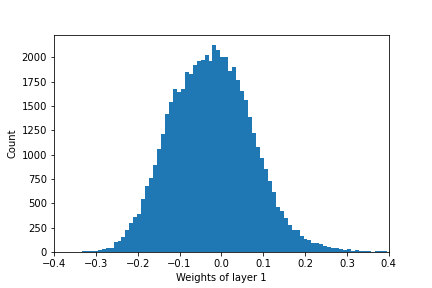}&
        \hspace{-1.5em}\includegraphics[width=0.35\textwidth]{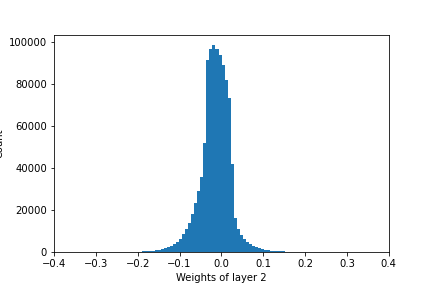}&
        \hspace{-1.5em}\includegraphics[width=0.35\textwidth]{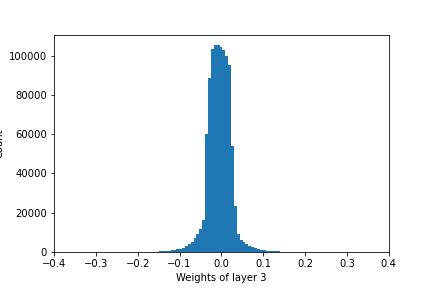}\\
        \hspace{-1.5em}\includegraphics[width=0.35\textwidth]{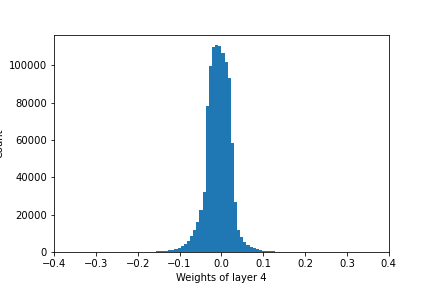}&
        \hspace{-1.5em}\includegraphics[width=0.35\textwidth]{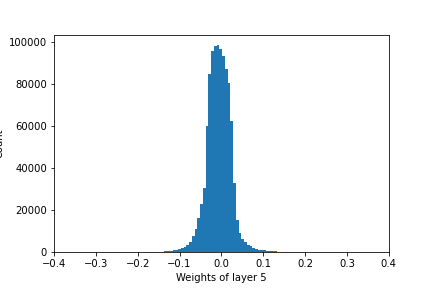}&
        \hspace{-1.5em}\includegraphics[width=0.35\textwidth]{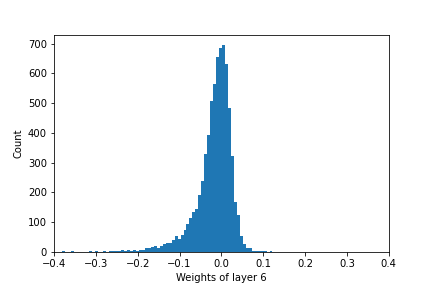}
    \end{tabular}
    \caption{\small{The histogram of entries of all the weight matrices from a trained FCN.}}
    \label{fig:hist}
\end{figure*}

\subsection{Further Discussions and Proof Outlines} \label{subsec:scratch_FCN}

In conclusion, under two different schemes, Theorems \ref{thm:FCN_uniform} and \ref{thm:FCN_general} show that, under certain conditions, we can prune $\lfloor D_k^{1-\alpha} \rfloor, \alpha > 0$ entries in the $k$-th layer of $F$ while keeping the pruned network $f$ to be $\epsilon$-close to $F$ with positive probability $1-\delta$. It is obvious that we cannot set $\alpha = 0$ as it makes the weight-subnetwork $f$ to be the zero function.

By fixing $\alpha$, our theorems show that a lower-bound of the minimum width of the target network can be represented as a polynomial in $\frac{1}{\epsilon}, \frac{1}{\delta} \textrm{ and } \log\pth{\frac{1}{\delta}}$.
Note that the constants in the theorems can be significantly improved by a finer analysis, but this is not the focus of this work. For example, by carefully discussing the independency of the events given in \eqref{eq:theorem1_events} and \eqref{eq:theorem2_events} in Appendix \ref{appendix:sec:proof}, we can improve the constants related to $\delta$ greatly. We can also give a finer upper-bound of the norm of the output of each layer by studying the corresponding distribution as a whole; in this case, we can even set the constant $C_2 = 1$ in Theorems \ref{thm:FCN_uniform} and \ref{thm:FCN_general}. The same argument holds for Theorem \ref{thm:CNN} (presented later) as well.

Next, we give a sketch on the universal framework for proving the theorems in this paper. For simplicity, we remove the statement about probabilities and use $C$ and $C'$ to denote universal positive constants which may vary by occurrence in this section. The details of the probabilities and constants are given in the full proof in Appendix \ref{appendix:sec:proof}. 

We use $y_k(x)$ and $y_k^*(x)$ to denote the output of the $k$-th layer of $f$ and $F$, respectively. The basic building block of the proof is showing how to iteratively bound the error between $y_k(x)$ and $y_k^*(x)$. This is achieved by inducting on the upper-bounds of $\norm{y_k(x) - y_k^*(x)}_2$ and $\norm{y_k^*(x)}_2$ at the same time. Intuitively, we expect the error $\norm{y_k(x) - y_k^*(x)}_2$ to be small and that the norm of the output $\norm{y_k^*(x)}_2$ is not exploding.

By the Lipschitz continuity of the activation functions and several matrix norm inequalities, we show that the above norms heavily depend on bounding two random variables  $\norm{W_k^*}_2$ and $\norm{W_k - W_k^*}_2$. 

Recall that we assume different distributions for the weights in the two theorems. For example, in Theorem \ref{thm:FCN_uniform} we assume the entries in $W_k^*$ are uniformly distributed. By Lemma \ref{lemma:2norm} and the Markov's inequality, we derive the probability that $\norm{W_k^*}_2 \le C$. A similar approach, depending on the specific distribution we assume, is applied in the other theorems to upper-bound the probability.

We want to make sure that $\norm{W_k - W_k^*}_2$ is small so that the gap between outputs can be small as well. In this sense, we cannot bound the two matrices $W_k$ and $W_k^*$ separately. Instead, we use the fact that $W_k - W_k^*$ is a zero matrix except for those pruned entries. In Theorem \ref{thm:FCN_uniform} we apply the closed-form order statistics of the uniform distribution to give a precise upper-bound of $\norm{W_k - W_k^*}_2$, which is $O(d^{-\alpha})$. In the other proofs, we rely on the results of the ``balls-into-bins'' problem (Lemma \ref{lemma:ballsinbins}) and the Latala's inequality (Lemma \ref{lemma:latala}) to give similar upper-bounds. 

The remaining part of the proofs are to estimate the probabilities that each event happens, and to determine the conditions between variables $d, \epsilon, \delta$ and $\alpha$. For more details about the proofs, please refer to Appendix \ref{appendix:sec:proof}.

\section{Pruning Convolutional Neural Networks} \label{sec:CNN}

In this section, we study model pruning of CNNs. We start with the mathematical definition of a single convolutional layer of CNN following the notations of \citet{sedghi2018singular}. We are given an input feature map $X \in \reals^{d \times p \times p}$ where $d$ denotes the number of input channels of the convolutional layer and $p$ is the height/width of the input feature map\footnote{We assume that the input feature map has the same width and height for simplicity. All the statement in this section can be generalized to fit different width and height.}. The entry $\pth{X}_{t,i,j}$ is the value of the input unit within channel $t$ at row $i$ and column $j$. The convolutional layer transforms $X$ into an output feature map $Y \in \reals^{d' \times p' \times p'}$, which becomes the input to the next convolutional layer. This is achieved by applying $d'$ 3D filters $\calF_{s} \in \reals^{d \times q \times q}$ on the $d$ input channels of $X$, where each $\calF_s$ generates the $s$-th channel of $Y, s \in [d']$, and $p > q$. Each filter $\calF_s $ is composed by $d$ 2D convolutional kernels (we use kernels for abbreviation in the sequel) $\calF_{s,t} \in \reals^{q \times q}, t \in [d]$. All the filters are combined to form the convolutional tensor $\calF \in \reals^{d'\times d \times q \times q}$. Mathematically, we have $\calF_{s,t} = \calF_{s,t,:,:}$ and $\calF_s = \calF_{s,:,:,:}, s \in [d'], t\in[d]$.

Filter $\calF_s$ is moved along the second and third axes of $X$ to get the output feature maps. We assume that the stride is 1, i.e., we move the filter $\calF_i$ by 1 pixel/element around every time. Note that there are two types of padding: (i) zero padding where we wrap the input feature maps with zeros around the edges; (ii) wrap-around padding where we pad the input feature maps in such a way that, if a pixel/element that is off the right end of the image is called by the filter, we use the pixel/element from the left end of the image instead; we do this similarly for all the edges and axes; mathematically, we set $ X_{t,i,j} = X_{t, i\%p, j\%p}$. Throughout the paper, we use the second approach for padding, as it leads to a circulant representation of the filters\footnote{The first approach leads to the Toeplitz representation and there exist numerous discussions regarding the error and (non-) asymptotic relationship between these two approaches in the CNN literature \citep{sedghi2018singular} and the matrix analysis literature \citep{CIT-006, zhu2017asymptotic}. The error gap can be bounded by $O\pth{\frac{1}{n}}$, where $n$ is the dimension of the matrix.}. 

With wrap-around padding and stride 1, the width and height of the output feature map are the same as the input feature map, i.e. we have $p = p'$. Let $K$ be the $d'\times d \times p \times p$ tensor such that {\footnotesize \begin{equation}
    K_{s,t,:,:} = \begin{bmatrix} \calF_{s,t,:,:} & \vec{0}_{q\times(p-q)} \\ \vec{0}_{(p-q)\times q} & \vec{0}_{(p-q)\times(p-q)}
\end{bmatrix}, s \in [d'], t \in [d]. \label{eq:calF}
\end{equation}}Then for $s\in[d'], a, b \in[p]$, we have \begin{equation*} 
    Y_{s,a,b} = \sum_{t\in[d]}\sum_{i\in[p]}\sum_{j\in[p]} X_{t, (a+i-1)\%p, (b+j-1)\%p} K_{s, t, i, j}.
\end{equation*}
For vector $a = (a_1, \ldots, a_n)^T$, we define {\footnotesize \begin{equation*}
    \textrm{circ}(a) \defeq \begin{bmatrix}
        a_1 & a_2 & \cdots & a_n \\
        a_n & a_1 & \cdots & a_{n-1} \\
        \vdots & \vdots & \ddots & \vdots \\
        a_2 & a_3 & \cdots & a_1
    \end{bmatrix}.
\end{equation*}}\citet{sedghi2018singular} show that a linear transformation $W \in \reals^{p^2d' \times p^2d}$ which satisfies $\textrm{vec}(Y) = W \textrm{vec}(X)$ can be represented by {\footnotesize \begin{equation}  \label{eq:CNN_linear_transformation}
    W = \begin{bmatrix}
        B_{1,1} & \cdots & B_{1, d} \\
        \vdots & \ddots & \vdots \\
        B_{d', 1} & \cdots & B_{d', d}
    \end{bmatrix},
\end{equation}}where each $B_{s,t}$ is a doubly block circulant matrix such that {\footnotesize \begin{equation} \label{eq:doubly_block}
    B_{s,t} = \begin{bmatrix}
            \textrm{circ}\pth{K_{s,t,1,:}} & \textrm{circ}\pth{K_{s,t,2,:}} & \cdots & \textrm{circ}\pth{K_{s,t,p,:}}\\
            \textrm{circ}\pth{K_{s,t,p,:}} & \textrm{circ}\pth{K_{s,t,1,:}} & \cdots & \textrm{circ}\pth{K_{s,t,p-1,:}} \\
            \vdots & \vdots & \ddots & \vdots \\
            \textrm{circ}\pth{K_{s,t,2,:}} & \textrm{circ}\pth{K_{s,t,3,:}} & \cdots &  \textrm{circ}\pth{K_{s,t,1,:}}
        \end{bmatrix}.
\end{equation}}Now we discuss the formulation of a convolutional neural network. Formally, consider a CNN $F$ of depth $l \ge 3$. For $1 \le k < l$, the $k$-th convolutional layer of $F$ takes the input feature map of dimension $d_{k-1} \times p_{k-1} \times p_{k-1}$, and transforms it to an output feature map of dimension $d_{k} \times p_{k} \times p_{k}$ by applying the convolutional tensor $\calF^{(k)} \in \reals^{d_k \times d_{k-1} \times q_{k-1} \times q_{k-1}}$. Then we pass the output feature map through an activation function $\sigma$ and feed it into the next layer. The last layer is a fully-connected layer that maps the output tensor of the previous layer with dimension $d_{l-1}\times p_{l-1} \times p_{l-1}$ into a vector of dimension $d_l$ by matrix $W_l^* \in \reals^{d_l \times d_{l-1} p_{l-1}^2}$. Mathematically, by reshaping the convolutional tensor $\calF^{(k)}$ into the corresponding linear mapping $W_k^* = \qth{B_{s,t}^{(k)}}_{s \in [d_k], t \in [d_{k-1}]}$, where $B_{s,t}^{(k)}$ is the doubly block circulant matrix induced by $K^{(k)}_{s,t,:,:}$ as defined in \eqref{eq:calF} -- \eqref{eq:doubly_block}, we write the convolutional neural network as \begin{equation} \label{eq:target_CNN_general_form}
    F(x) = W_l^* \sigma \pth{W_{l-1}^* \sigma_{l} \pth{\cdots W_2^* \sigma \pth{W_1^* x}}}.
\end{equation}
Similar to the definition of weight- and neuron-subnetworks of FCN, there are two analogous definitions for CNN. We define the channel-subnetwork of $F$ as achieved by removing several 3D channels from the 4D tensor $\calF^{(k)}, 1 \le k < l$\footnote{In practice we usually remove the whole channel and hence reduce the size of $\calF^{(k)}$ to $\reals^{d_k' \times d_{k-1} \times q_{k-1} \times q_{k-1}}$ with $d_k' < d_k$. The size of the input of the next layer is also reduced to $d_k' \times p_k \times p_k$. In the presentation of this paper, we set the pruned channels to zero instead of removing them. It helps us to keep the dimension of pruned and original tensors to be the same while not changing any theoretical property of the CNNs.}, and the filter-subnetwork of $F$ by removing several 2D filters from the 4D tensor $\calF^{(k)}$. The channel-subnetwork of $F$ is equivalent to setting rows of $W_k^*$ in terms of equation \eqref{eq:target_CNN_general_form} (we are actually setting several rows of block matrices $B_{s,t}^{(k)} $) to be zero while the filter-subnetwork refers to setting some block sub-matrices $B_{s,t}^{(k)}$ of $W_k^*$ to be zero. In the following, we focus on filter-subnetworks and present the result of random pruning on CNNs. We discuss the magnitude-based pruning of CNNs in Appendix \ref{appendix:magnitude:CNN}.

For ease of presentation, in the theorem below, we assume that the number of channels and the width/height of each channel in all convolutional layers are equal, i.e., we define $d \defeq d_1 = \cdots = d_{l-1}$ and $p \defeq p_1 =\cdots p_{l-1}$. A similar result can be derived by the same approach for the general non-homogeneous setting.

\begin{theorem} \label{thm:CNN}
    We are given a target network $F$ as defined in \eqref{eq:target_CNN_general_form} and we denote by $\calF^{(k)} \in \reals^{d_k \times d_{k-1} \times p_k \times p_k}$ the convolutional tensor corresponding to $W_k^*, 1 \le k < l$. Let us assume that \begin{enumerate}[(i)]
        \vspace{-1em}\item $\sigma$ is $L$-Lipschitz and $\sigma(0) = 0$;
        \vspace{-0.5em}\item $d \ge \max\sth{d_0, d_l}$;
        \item for $s\in[d_k],t\in[d_{k-1}],i,j\in[p],k \in [l-1]$, $\calF^{(k)}_{s,t,i,j}$ independently follows a distribution $\calX^k_{s,t,i,j}$; further, there exist two positive constants $C_1$ and $C_2$ such that $\mathbb{E} {\calX^k_{s,t,i,j}}=0$, $\mathbb{E} \abs{\calX^k_{s,t,i,j}}^2 \le \frac{C_1}{p^2d}$ and $\mathbb{E} \abs{\calX^k_{s,t,i,j}}^4 \le \frac{C_2}{p^4d^2}$; the weights in $W_{l}^*$ follow distributions with the same second-order and fourth-order moment upper-bounds.
    \end{enumerate} Let us consider the subnetwork $f$ of $F$ with mask $M = \sth{M_1, \ldots, M_l}$ that randomly prunes $\lfloor d^{2-\alpha} \rfloor, 0 < \alpha \le 2 - \frac{\log(d+1)+\log^{(2)}(d)}{\log(d)} $ filters in the $k$-th layer of $F$, $1<k<l$. For any positive constants $\beta_1 \in (0,1)$ and $\beta_2 \in \pth{0, \frac{1}{4}\alpha}$, with probability at least $\pth{1-d^{-\frac{1}{3}}}^{2(l-2)}\widebar{p}$, where $\widebar{p} \defeq 1 - (l-2)C_4 \frac{q^2}{p}d^{-\frac{1}{4}\alpha + \beta_2} - \frac{l^2-l-2}{2}C_3\frac{q^2}{p^{1-\beta_1}} - \frac{C_5}{p^{1-\beta_1}}$ over the randomness of masks and weights, we have {\footnotesize \begin{align}
        \sup_{x \in \calC_{p_0^2d_0}} &\norm{f(x) - F(x)}_2 \label{eq:CNN_upperbound}\\ 
        &\le p^{-\beta_1} L^{l-1} p_0 \sqrt{d} \qth{ p^{-\beta_1} \pth{p^{-\beta_1} + d^{-\beta_2}}^{l-2} - p^{-(l-1)\beta_1}} \notag
    \end{align}}for some positive constants $C_3, C_4$ and $C_5$ specified in the proof. 
\end{theorem}

The first two assumptions are common in all the theorems we present. We next discuss the feasibility of assumption (iii). Since we translate the target CNN into a FCN form and there are $p^2 d$ neurons (instead of $d$ neurons) in the $k$-th layer of $F, 1<k<l$, we change the denominators in the upper-bounds of moments accordingly. The mathematical definition of the masks is also revised to fit the CNN structure. Here we set the mask $M_k$ to be the 0-1 matrices such that its sub-matrices are blocks of the zero matrices and all one matrices based on \eqref{eq:CNN_linear_transformation}. 
Condition $\alpha \le 2 - \frac{\log(d+1)+\log^{(2)}(d)}{\log(d)}$ is used to guarantee that Lemma \ref{lemma:ballsinbins} holds. As $d$ goes to infinity, the left-hand side goes to 1 and then $\alpha$ becomes less and less constrained. For example, for $d = 128$ or $1,024$, the bound reads $0.6729, 0.7205$, respectively. 

We next argue that the probability $\pth{1-d^{-\frac{1}{3}}}^{2(l-2)}\widebar{p}$ is positive in many of the use cases. Note that $\frac{q^2}{d}$ and $ \frac{q^2}{p}$ are close to zero as we usually take $q=1,3,5$ as the dimension of the kernel, $p=28$ for images of MNIST \citep{lecun1998gradient} and $p = 32$ for images of CIFAR-10 \citep{krizhevsky2009learning}. For images with 4K resolution, we have $p= 3,840 \textrm{ or } 2,160$. The number of channels $d$ varies from 64 to 512 in famous CNN architectures, like VGG16 \citep{simonyan2014very} and ResNet \citep{he2016deep}. The closed-forms of constants $C_3, C_4$, and $C_5$ are presented in Appendix \ref{appendix:proof3}. Similar to the discussions in Section \ref{subsec:scratch_FCN}, these constants can be significantly improved by a finer analysis, but this is not the focus of this work. It is easy to see that the right-hand side of \eqref{eq:CNN_upperbound} is positive and it converges to 0 as $d$ goes to infinity. Thus, by taking $\beta_1$ and $\beta_2$ appropriately small and $d$ to be large, we can make sure that the probability $\pth{1-d^{-\frac{1}{3}}}^{2(l-2)}\widebar{p}$ is positive while the upper-bound of the gap between the pruned and target networks is small. {We also point out that $\pth{1-d^{-\frac{1}{3}}}^{2(l-2)}$ is the probability with respect to masks and $\widebar{p}$ with respect to weights. As a result, the statement holds ``for almost all masks.''}

\section{Discussion and Future Works} \label{sec:conclusion}

In this paper, we establish theoretical results of model pruning for FCNs and CNNs under different schemes with mild assumptions. For magnitude-based pruning, we show the sub-network $f$ of $F$, which prunes $\lfloor D_k^{1-\alpha} \rfloor$ out of $D_k$ smallest entries of the $k$-th layer of $F$, can approximate the expressive power of $F$ on the unit ball or the unit cube with positive probability. For random pruning, we show that most random masks, which prune $\lfloor D_k^{1-\alpha} \rfloor$ out of $D_k$ entries of the $k$-th layer of $F$, approximate the expressive power of $F$ on the unit ball or the unit cube with positive probability. Our results are enabled by many results from the random matrix theory. The essential building block of our analysis is to iteratively show that the gap between the pruned and target weight matrices and the gap between the outputs of the $k$-th layer of the pruned and target networks are small.

{This is one of the rare theoretical works that discusses pruning of FCNs and CNNs.} We not only cover model pruning of general FCNs, but also establish the results regarding pruning CNNs. {The results can be applied to a variety of other network structures given the fact that almost all networks can be represented by a stack of fully-connected layers.}
Our theorems can provide precious insights to the iterative magnitude-based pruning as suggested by \citet{frankle2018lottery}. For example, our results are able to determine how many weights we can prune in each iteration and the corresponding probability that the gap between the pruned and target networks is smaller than a given error.

As discussed in Appendix \ref{appendix:sec:magnitude}, a direct extension of this work is to consider magnitude-based pruning for general distributions. {Appendix \ref{appendix:subsec:independent} discusses the assumption about the independency of weights in the target network and provide many approaches to alleviate it, but a detailed and strict theoretical study is definitely attractive.}
Besides, we usually use pooling layers and residual connections in practical CNN models. It is interesting to consider the impact of such non-parametric functions and skip connections on the theoretical neural network pruning results. 
Another interesting problem is trying to leverage additional information (e.g., gradients) of the target network to improve our results. Besides, it would be exciting if our results can provide guidance to improve the existing magnitude-based and random pruning algorithms. 

\section*{Acknowledgement} The authors would like to thank Yiqiao Zhong, Dawei Li, and Feiyi Xiao for critically reading the manuscript and helpful discussions.

\bibliography{ref}
\bibliographystyle{icml2021}


\newpage
\appendix
\onecolumn

\section{Supporting Lemmas} \label{appendix:sec:lemma}

We start by presenting various technical lemmas that support the main proofs. Lemma \ref{lemma:squaredUniform_single} shows the expectation of moments of order statistics of the uniform distribution. This lemma is used in the magnitude-based pruning result of FCNs.

\begin{lemma} \label{lemma:squaredUniform_single}
	Given $n$ independent and identically distributed random variables $U_1, \ldots, U_n \sim \calU\qth{-a, a}$ and $X_i = U_i^2, i \in [n]$, we have \begin{align*}
		\Expect X_{(r)} = a^2 \frac{(r+1)r}{(n+2)(n+1)} \quad \textrm{and} \quad
		\Expect X_{(r)}^2 = a^4 \frac{(r+3)(r+2)(r+1)r}{(n+4)(n+3)(n+2)(n+1)},
	\end{align*}
	where $r \le n$ is a constant and $X_{(1)} \le  \cdots \le X_{(n)} $ are order statistics of $X_1, \ldots, X_n$. 
\end{lemma}

\begin{proof}
	Note that for $0\le x \le a^2$, we have $F(x) \triangleq \prob{X_i \le x} = \prob{U_i^2 \le x} = \prob{-\sqrt{x} \le U_i \le \sqrt{x}} = \frac{\sqrt{x}}{a}$. Therefore, the probability density function of $X_{(r)}$ is given by \begin{align*}
		f_{(r)}(x) &= \frac{n!}{(r-1)!(n-r)!} \qth{F(x)}^{r-1}\qth{1-F(x)}^{n-r} F'(x)\\
		&= \frac{n!}{(r-1)!(n-r)!} \qth{\frac{\sqrt{x}}{a}}^{r-1}\qth{1-\frac{\sqrt{x}}{a}}^{n-r} \frac{1}{2a\sqrt{x}}.
	\end{align*}
	
	For $p \in \mathbb{Z}$, we have\begin{align*}
		\Expect X_{(r)}^p &= \int_0^{a^2} x^p f_{(r)}(x) \diff x \\
		&= \int_0^{a^2} x^p \frac{n!}{(r-1)!(n-r)!} \qth{\frac{\sqrt{x}}{a}}^{r-1}\qth{1-\frac{\sqrt{x}}{a}}^{n-r} \frac{1}{2a\sqrt{x}} \diff x  \\
		&= \frac{n!}{(r-1)!(n-r)!} \frac{1}{2a^n} \int_0^{a^2} x^p x^{\frac{r}{2}-1} (a-\sqrt{x})^{n-r} \diff x\\
		&= \frac{n!}{(r-1)!(n-r)!} \frac{1}{2a^n} \int_0^{1} (at)^{2p+r-2} (a-at)^{n-r} 2a^2 t \diff t \\
		&= \frac{a^{2p} n!}{(r-1)!(n-r)!} \int_0^{1} t^{r+2p-1} (1-t)^{n-r} \diff t \\
		&= \frac{a^{2p} n!}{(r-1)!(n-r)!} \frac{ (n-r)! (r+2p-1)!}{(n+2p)!} \\
		&= \frac{(r+2p-1)!n!}{(r-1)!(n+2p)!} a^{2p}.
	\end{align*}
	
	Specifically, we have $\Expect X_{(r)} = a^2 \frac{(r+1)r}{(n+2)(n+1)}$ and $\Expect X_{(r)}^2 = a^4 \frac{(r+3)(r+2)(r+1)r}{(n+4)(n+3)(n+2)(n+1)}$.
\end{proof}

Next, we present some results for sub-Gaussian random matrices. We first give the definition of sub-Gaussian random variables in the following.

\begin{definition}
    A random variable $X \in \reals$ is said to be sub-Gaussian with variance proxy $\sigma^2$ if $\Expect X = 0$ and its moment generating function satisfies \begin{equation*}
        \Expect\exp\qth{sX} \le \exp\pth{\frac{\sigma^2 s^2}{2}}, \quad \forall s \in \reals.
    \end{equation*}
    In this case, we write $X \sim \textsf{subG}(\sigma^2)$.
\end{definition}

Note that $\textsf{subG}(\sigma^2)$ denotes a class of distributions rather than a single distribution. Many common distributions, like Gaussian and any bounded distributions with zero expectation, all fall into this category. If $X \sim \textsf{subG}(\sigma^2)$, then we have $\var\pth{X} = \Expect X^2 \le \sigma^2$.

\begin{lemma}[Proposition 2.4 of \citet{rudelson2010non}] \label{lemma:sub-Gaussian_random_matrix}
	Let $A$ be a $n_1 \times n_2$ random matrix whose entries are independent mean zero sub-Gaussian random variables whose sub-Gaussian variance proxy are bounded by $1$. Then there exists universal positive constants $c$ and $C$ such that, for any $t > 0$ we have \begin{equation}
		\prob{\norm{A}_2 > C\pth{\sqrt{n_1}+\sqrt{n_2}} + t} \le 2e^{-ct^2}.
	\end{equation}
\end{lemma}

\begin{lemma} \label{lemma:2norm}
	Let $B$ be a $n_1 \times n_2$ random matrix whose entries are independently and identically distributed following $\calU\qth{-\frac{K}{\sqrt{n}}, \frac{K}{\sqrt{n}}}$, where $K$ is a positive constant and $n = \max\sth{n_1, n_2}$. Then there exist positive constants $c_0$ (depends on $K$) and $\delta_0$ such that  $\norm{B}_2 \le c_0$ with probability at least $1-2e^{-4\delta_0 n}$. 
\end{lemma}

\begin{proof}
    Let us denote $A = \frac{\sqrt{3n}}{K} B$. Then the entries in $A$ are independently and identically distributed following $\calU\qth{-\sqrt{3}, \sqrt{3}}$, which belongs to the sub-Gaussian distribution with variance proxy 1. Applying Lemma \ref{lemma:sub-Gaussian_random_matrix}, we know that there exist positive constants $C$ and $\delta_0$ such that \begin{align*}
		\prob{\norm{A}_2 > 2C\sqrt{n} + t} \le \prob{\norm{A}_2 > C\pth{\sqrt{n_1} + \sqrt{n_2}} + t} \le 2e^{-\delta_0 t^2}.
	\end{align*}
	
	Taking $t = 2\sqrt{n}$, we have \begin{align*}
		\prob{\norm{B}_2 > \frac{2K}{\sqrt{3}}\pth{C+1}} = \prob{\frac{K}{\sqrt{3n}} \norm{A}_2 > \frac{2K}{\sqrt{3}}\pth{C+1} } = \prob{\norm{A}_2 > 2\sqrt{n}\pth{C + 1}} \le 2e^{-4\delta_0 n},
	\end{align*}
	and therefore \begin{align*}
		\mathbb{P}\pth{\norm{B}_2 \le c_0} \ge 1 - 2e^{-4\delta_0 n},
	\end{align*}
	where $c_0 = \frac{2K}{\sqrt{3}}\pth{C+1} > 0$.
\end{proof}

In the two lemmas above, we assume certain distributions for the entries in the random matrices. The following lemma is more general in the sense that it only requires the entries in the matrices to be independent.

\begin{lemma}[Theorem 2 of \citet{latala2005some}] \label{lemma:latala}
    Let $A$ be a random matrix whose entries $A_{i,j}$ are independent mean zero random variables with finite fourth moment. Then \begin{equation}
        \Expect\norm{A}_2 \le C \qth{\max_i \pth{\sum_j \Expect A_{i, j}^2}^{\frac{1}{2}} + \max_j \pth{\sum_i \Expect A_{i, j}^2}^{\frac{1}{2}} + \pth{\sum_{i,j}\Expect A_{i, j}^4}^{\frac{1}{4}}},
    \end{equation}
    where $C$ is an universal positive constant.
\end{lemma}

The proofs of the main theorems in this paper heavily rely on Lemmas \ref{lemma:2norm} and \ref{lemma:latala}. Note that there are some universal constants in the statement of these two lemmas that all appear in the bounds of the main theorems. Thus we give a numerical study of these two lemmas in Appendix \ref{appendix:subsec:uniform} and \ref{appendix:subsec:latala}.


\begin{lemma}[Chernoff Bound] \label{lemma:chernoff_bound}
    Suppose $X_1, \ldots, X_m$ are independent random variables taking values in $\sth{0, 1}$. Let $X \defeq \sum_{i=1}^m X_i$ and $\mu \defeq \Expect X$. Then for any $\delta > 0$, we have \begin{equation}
        \prob{X \ge (1+\delta)\mu } \le \exp\pth{-\frac{\delta^2}{1+\delta}\mu}.
    \end{equation}
\end{lemma}

The next lemma results from the famous problem ``balls-into-bins.'' This is a classic problem in probability theory that has many applications in computer science. See the survey paper by \citet{richa2001power} for more details.

\begin{lemma} \label{lemma:ballsinbins}
    Consider the problem of throwing $N$ balls independently and uniformly at random into $n$ bins. Let $X_j$ be the random variable that counts the number of balls in the $j$-th bin, $1 \le j \le n$. If $N \ge n \log(n)$, then with probability at least $1-n^{-\frac{1}{3}}$ we have $\max_{j\in[n]} X_j \le \frac{3N}{n}$.
\end{lemma}

\begin{proof}
    Let $X_{ij}$ be the indicator random variable for the event that the $i$-th ball falls into the $j$-th bin, $i\in [N], j \in [n]$. Then $\Expect X_j = \sum_{i=1}^N \Expect X_{ij} = \frac{N}{n}, j \in [n]$. Note that $\mu = \frac{N}{n} \ge \log(n)$, applying Lemma \ref{lemma:chernoff_bound} with $\delta = 2$, we have \begin{align*}
        \prob{X_j \ge 3\frac{N}{n}} \le \exp \pth{-\frac{4}{3}\mu} \le \exp \pth{-\frac{4}{3}\log(n)} = n^{-\frac{4}{3}}.
    \end{align*}
    By the union bound we have \begin{align*}
        \prob{\max_{j\in[n]} X_j\ge \frac{3N}{n}} = \prob{\bigcup_{j\in[n]} \sth{X_j\ge \frac{3N}{n}}}\le \sum_{j\in[n]} \prob{X_j\ge \frac{3N}{n}} \le n \cdot n^{-\frac{4}{3}} = n^{-\frac{1}{3}},
    \end{align*}
    and therefore \begin{align*}
        \prob{\max_{j\in[n]} X_j \le \frac{3N}{n}} \ge 1 - n^{-\frac{1}{3}}.
    \end{align*}
\end{proof}

The last lemma focuses on the singular values of the matrix representation of convolutional operators. Given a convolutional tensor $\calF \in \reals^{d\times d \times q\times q}$, the corresponding matrix representation $W$ of $\calF$ has dimension $p^2d \times p^2d$, where $p$ is the width and height of the input feature map. Applying the traditional singular value decomposition methods on such a large matrix is usually time-consuming and computationally-inefficient. \citet{sedghi2018singular} provide tools to represent the set of singular values of $W$ by the joint of sets of singular values of many smaller sub-matrices. This is done by carefully analyzing the properties of ciuculant-type matrices. We use the following lemma from \citet{sedghi2018singular} to calculate the $L_2$ norm of the weight matrices in CNNs. 

\begin{lemma}[Theorem 6 of \citet{sedghi2018singular}] \label{lemma:sedghi}
    Let $\omega = \exp\pth{2\pi i/p}$, where $i=\sqrt{-1}$ and $S$ be the $p \times p$ matrix that represents the discrete Fourier transform $$S \defeq \begin{bmatrix}
        \omega^{1\times 1} & \cdots & \omega^{1\times p} \\
        \vdots & \ddots & \vdots \\
        \omega^{p \times 1} & \cdots & \omega^{p \times p}
    \end{bmatrix}.$$
    Given a tensor $\calF \in \reals^{d\times d \times q\times q}$, let us denote $K \in \reals^{d\times d \times p \times p}$ as defined in \eqref{eq:calF} and we denote $W^* \in \reals^{dp^2 \times dp^2}$ as the matrix encoding the linear transformation computed by the convolutional layer parameterized by $K$, as defined in \eqref{eq:CNN_linear_transformation} -- \eqref{eq:doubly_block}. Let $P^{(u,v)}$ be the $d \times d$ matrix such that the $(s, t)$-th element of $P^{(u,v)}$ is equal to the $(u, v)$-th element of $S^T K_{s,t,:,:} S, u, v \in [p], s, t, \in [d]$, or equivalently \begin{equation*}
        P^{(u,v)}_{s,t} = \pth{S^T K_{s,t,:,:} S}_{u,v}, \quad u, v \in [p], s, t, \in [d].
    \end{equation*} Then \begin{equation*}
        \norm{W^*}_2 = \max_{u,v \in [p]} \sth{\norm{P^{(u,v)}}_2}.
    \end{equation*}
\end{lemma}

\section{Proofs} \label{appendix:sec:proof}

In this section, we provide the full proof of Theorems \ref{thm:FCN_uniform}, \ref{thm:FCN_general}, and \ref{thm:CNN}. Note that the proofs of these three theorems are similar. Theorem \ref{thm:FCN_general} exhibits all ideas and thus it is presented in full. The proofs of the other theorems show the difference.

\subsection{Proof of Theorem \ref{thm:FCN_general}}

\begin{proof}
    For any $x \in \calB_{d_0}$ and $1 \le k < l$, we denote $y_k(x) \defeq \sigma_k \pth{W_k \sigma_{k-1} \pth{\cdots W_2 \sigma_1 \pth{W_1 x}}}$ and $y_k^*(x) \defeq \sigma_k\pth{W_k^* \sigma_{k-1} \pth{\cdots W_2^* \sigma_1 \pth{W_1^* x}}}$ as the output of the $k$-th layer of $f$ and $F$, respectively. 
    
    Recall that we set $M_1$ and $M_l$ as the all 1 matrices, i.e. $W_1 = W_1^*$ and $W_l = W_l^*$. For each $1 < k < l$, we order the entries of $W_k^*$ by their absolute values such that \begin{equation*}
        \abs{(W_k^*)_{i_1^k, j_1^k}} \le \abs{(W_k^*)_{i_2^k, j_2^k}} \le \cdots \le \cdots \le \abs{(W_k^*)_{i_{D_k}^k, j_{D_k}^k}}
    \end{equation*}
    and denote $\calI_k \defeq \sth{ (i_s^k, j_s^k): 1 \le s \le \lfloor D_k^{1-\alpha} \rfloor }$. We set $(M_k)_{i,j} = 0 $ if $(i,j) \in \calI_k$, and $(M_k)_{i,j} = 1 $ otherwise. We further denote two events \begin{align*}
        A^{(k)}_{r} \defeq \sth{\textrm{the number of zero entries in each row of } M_k \textrm{ is at most } 3\lfloor D_k^{1-\alpha}\rfloor/d_{k}}, \\  
        A^{(k)}_{c} \defeq \sth{\textrm{the number of zero entries in each column of } M_k \textrm{ is at most } 3\lfloor D_k^{1-\alpha}\rfloor/d_{k-1}}
    \end{align*} and set event $A^{(k)} \defeq A^{(k)}_{r} \bigcap A^{(k)}_{c}$. Note that \eqref{eq:assumption_balls_in_bins_for_alpha_1} and \eqref{eq:assumption_balls_in_bins_for_alpha_2} guarantee that $\lfloor D_k^{1-\alpha} \rfloor \ge d_{k}\log\pth{d_{k}}$ and $D_k^{1-\alpha} \ge d_{k-1}\log\pth{d_{k-1}}$, respectively, and the events $A^{(k)}_{r} $ and $A^{(k)}_{c}$ are independent. Thus by Lemma \ref{lemma:ballsinbins}, we have \begin{equation*}
        \prob{A^{(k)}} = \prob{A^{(k)}_{r} \bigcap A^{(k)}_{c}} = \prob{A^{(k)}_{r}}\prob{A^{(k)}_{c}} \ge\pth{1-d_{k}^{-\frac{1}{3}}}\pth{1-d_{k-1}^{-\frac{1}{3}}} \ge \pth{1-d^{-\frac{1}{3}}}^2.
    \end{equation*}
    Further, for $A \defeq A^{(2)} \bigcap \cdots \bigcap A^{(l-1)}$, we have $\prob{A} = \prod_{k=2}^{l-1} \prob{A^{(k)}} \ge \pth{1-d^{-\frac{1}{3}}}^{2(l-2)}$ where the probability is taken over the randomness of masks (and is not over the randomness of weights in $W_k^*$'s).
    
    Let us assume that \begin{equation} \label{eq:FCN_general_d_assumption}
        d^{-\frac{1}{4}\alpha} \le \min\sth{N_2, \ldots, N_{l-1}, \frac{\epsilon}{\pth{2^{l-2}-1} L_{1:l-1} N_{1:l}}}.
    \end{equation} We use induction to show that, for any $x \in \calB_{d_0}$ and $1\le k < l$, \begin{enumerate}[(I)]
        \item with probability at least $\prod_{i=1}^{k}\pth{1 - \delta_i}$, we have $\norm{y_k^*(x)}_2 \le L_{1:k} N_{1:k}$,
        \item with probability at least $1 - (k-1) c_2 d^{-\frac{\alpha}{4}} - 2(k-1) d^{-\frac{1}{3}} - \sum_{i=1}^{k} (k+1-i)\delta_i$, we have $\norm{{\pth{y_k(x) \big| A}} - y_k^*(x)}_2 \le \pth{2^{k-1}-1} d^{-\frac{1}{4}\alpha} L_{1:k} N_{1:k}$ for some positive constant $c_2$ specified later\footnote{Note that in the induction statement (II), the probability (and the expectations in the following context) is taken over the randomness of weights but not the masks. The random variable $\norm{\pth{y_k(x) \big| A} - y_k^*(x)}_2$ is equivalent to $\norm{y_k(x) - y_k^*(x)}_2 \Big| A$, and the statement can also be written as $\prob{\sth{\norm{y_k(x) - y_k^*(x)}_2 \le \pth{2^{k-1}-1} d^{-\frac{1}{4}\alpha} L_{1:k} N_{1:k}} \Big| A} \ge 1 - (k-1) c_2 d^{-\frac{\alpha}{4}} - 2(k-1) d^{-\frac{1}{3}} - \sum_{i=1}^{k} (k+1-i)\delta_i$.}.
    \end{enumerate}
    
    The case of $k=1$ is as follows. Note that for any vector $v$, we have $\norm{\sigma_1(v)}_2 = \norm{\sigma_1(v)- \sigma_1(0)}_2 \le L_1 \norm{v - 0}_2 = L_1\norm{v}_2$. Thus, $\norm{y_1^*(x)}_2 =\norm{\sigma_1\pth{W_1^*x}}_2 \le L_1\norm{W_1^*x}_2 \le L_1 \norm{W_1^*}_2 \norm{x}_2 \le L_1 N_1 $ with probability at least $1 - \delta_1$. Further, we have $y_1(x) = \sigma_1\pth{W_1x} = \sigma_1\pth{W_1^*x} = y_1^*(x)$, and thus $\norm{y_1(x) - y_1^*(x)}_2 = 0$.
    
    Suppose the statement holds for $1 \le k < l-1 $; we consider the case of $k+1$. Note that the events $\sth{\norm{W_{k+1}^*}_2 \le N_{k+1}}$ and $\sth{\norm{y_{k}^*(x)}_2 \le L_{1:k} N_{1:k}}$ are independent. By induction statement (I), with probability at least $$\prob{\sth{\norm{W_{k+1}^*}_2 \le N_{k+1}} \bigcap \sth{\norm{y_{k}^*(x)}_2 \le L_{1:k} N_{1:k}}} = \prob{\norm{W_{k+1}^*}_2 \le N_{k+1}} \cdot \prob{\norm{y_{k}^*(x)}_2 \le L_{1:k} N_{1:k}} \ge \prod_{i=1}^{k}\pth{1 - \delta_i},$$ we have \begin{align*}
    	\norm{y_{k+1}^*(x)}_2 &= \norm{\sigma_{k+1} \pth{W_{k+1}^* y_k^*(x)}}_2 \le L_{k+1} \norm{W_{k+1}^* y_k^*(x)}_2 \le L_{k+1} \norm{W_{k+1}^* }_2 \norm{y_k^*(x)}_2 \\
    	&\le L_{k+1} N_{k+1} \cdot L_{1:k} N_{1:k} = L_{1:k+1} N_{1:k+1},
    \end{align*} which shows (I) in the induction statement. 
    
    
    
    
    We next show that (II) holds. Under event $A$, the number of non-zero entries in each row of $\calW_{k+1}$ is at most $3 \lfloor D_{k+1}^{1-\alpha}\rfloor/d_{k+1}$. Thus we have \begin{align} \label{eq:latala1}
        \max_{i\in [d_{k+1}]} \pth{\sum_{j\in[d_k]} \condexp{}{\pth{\calW_{k+1}}_{i, j}^2}{A}}^{\frac{1}{2}} \le \pth{ \frac{3\lfloor D_{k+1}^{1-\alpha} \rfloor}{d_{k+1}} \cdot \frac{K_1}{\max\sth{d_{k+1}, d_k}} }^{\frac{1}{2}} \le \pth{ \frac{3K_1D_{k+1}^{1-\alpha}}{d_{k+1}d_k}}^{\frac{1}{2}} = \sqrt{3K_1} D_{k+1}^{-\frac{\alpha}{2}} \le \sqrt{3K_1} d^{-\alpha},
    \end{align} and similarly, \begin{align} \label{eq:latala2}
        \max_{j\in [d_k]} \pth{\sum_{i\in[d_{k+1}]} \condexp{}{\pth{\calW_{k+1}}_{i, j}^2}{A}}^{\frac{1}{2}} \le \sqrt{3K_1} d^{-\alpha}.
    \end{align}

    In addition, since there are at most $\lfloor D_{k+1}^{1-\alpha}\rfloor$ non-zero entries in $\calW_{k+1}\defeq W_{k+1} - W_{k+1}^*$, we have \begin{align} \label{eq:latala3}
        \sum_{i\in [d_{k+1}],j\in [d_k]} \condexp{}{\abs{\pth{\calW_{k+1}}_{i,j}}^4}{A} \le \lfloor D_{k+1}^{1-\alpha}\rfloor \cdot \frac{K_2}{\max\sth{d_{k+1}, d_k}^2} \le D_{k+1}^{1-\alpha} \frac{K_2}{D_{k+1}} = K_2 D_{k+1}^{-\alpha} \le K_2 d^{-2\alpha}.
    \end{align}
    
    Combining \eqref{eq:latala1} -- \eqref{eq:latala3} and Lemma \ref{lemma:latala}, there exists a universal positive constant $c_1$ such that
    \begin{align}
        \condexp{}{\norm{\calW_{k+1}}_2}{A} &\le c_1 \qth{\max_{i\in [d_{k+1}]} \pth{\sum_{j\in [d_k]} \Expect \pth{\calW_{k+1}}_{i, j}^2}^{\frac{1}{2}} + \max_{j\in [d_k]} \pth{\sum_{i\in [d_{k+1}]} \Expect \pth{\calW_{k+1}}_{i, j}^2}^{\frac{1}{2}} + \pth{\sum_{i\in [d_{k+1}],j\in [d_k]}\Expect \pth{\calW_{k+1}}_{i, j}^4}^{\frac{1}{4}}} \notag\\
        &\le c_1 \qth{\sqrt{3K_1} d^{-\alpha} + \sqrt{3K_1} d^{-\alpha} + \pth{K_2 d^{-2\alpha}}^{\frac{1}{4}}} \label{eq:apply_latala}\\
        &\le c_2 d^{-\frac{\alpha}{2}} \notag,
    \end{align} where $c_2 = c_1\pth{2\sqrt{3K_1} + K_2^\frac{1}{4}}$. 
    
    
    By the Markov's inequality, for all $t > 0$ we have \begin{align*}
        \prob{\sth{\norm{\calW_{k+1}}_2 \ge t} \Big| A} \le \frac{\condexp{}{\norm{\calW_{k+1}}_2}{A}}{t}.
    \end{align*}
    Taking $t = d^{-\frac{\alpha}{4}}$, we have \begin{align*}
        \prob{\sth{\norm{\calW_{k+1}}_2 \le d^{-\frac{\alpha}{4}}} \Big| A} \ge 1 - c_2 d^{-\frac{\alpha}{4}}.
    \end{align*}

	By induction statement (II) and the fact that $\prob{\bigcap_{i=1}^s A_i} \ge \sum_{i=1}^s\prob{A_i} - (s-1)$, with probability at least\footnote{We use the fact that, for any $a_1, \cdots, a_s \in (0, 1)$, we have $\prod_{i=1}^s (1-a_i) \ge 1 - \sum_{i=1}^s a_i$. This inequality is frequently used in the following proofs.}  \begin{align}
	    & \; \mathbb{P} \bigg( \sth{\norm{\pth{W_{k+1} \big| A} - W_{k+1}^*}_2 \le d^{-\frac{\alpha}{4}}} \bigcap \sth{\norm{W_{k+1}^*}_2 \le N_{k+1}} \bigcap \sth{\norm{y_k^*(x)}_2 \le L_{1:k} N_{1:k}} \notag\\
	    & \quad \quad \quad\quad \quad \quad\quad \quad \quad\quad \quad \quad\quad \quad \quad \quad \bigcap \sth{\norm{\pth{y_k(x)\big| A} - y_k^*(x)}_2 \le \pth{2^{k-1}-1} d^{-\frac{1}{4}\alpha} L_{1:k} N_{1:k}} \bigg) \label{eq:induct_2_1} \\
	    \ge & \; \pth{1 - c_2 d^{-\frac{\alpha}{4}}} + \pth{1-\delta_{k+1}} + \prod_{i=1}^k \pth{1-\delta_i} + \pth{1 - (k-1) c_2 d^{-\frac{\alpha}{4}} - \sum_{i=1}^{k} (k+1-i)\delta_i} - 3 \notag \\
	    \ge & \; \pth{1 - c_2 d^{-\frac{\alpha}{4}}} + \pth{1-\delta_{k+1}} +  \pth{1-\sum_{i=1}^k \delta_i} + \pth{1 - (k-1) c_2 d^{-\frac{\alpha}{4}} - \sum_{i=1}^{k} (k+1-i)\delta_i} - 3 \notag \\
	    = & \; 1 - k c_2 d^{-\frac{\alpha}{4}} - \sum_{i=1}^{k+1} (k+2-i)\delta_i, \notag 
	\end{align} we have \begin{align}
		&\norm{{\pth{y_{k+1}(x) \big| A}} - y_{k+1}^*(x)}_2 \notag\\
        =\; & \norm{\sigma_{k+1} \pth{{\pth{W_{k+1}\big| A}} y_k(x)} - \sigma_{k+1} \pth{W_{k+1}^* y_k^*(x)}}_2 \notag\\
		\le\;& L_{k+1} \norm{\pth{W_{k+1}\big| A} y_k(x) - W_{k+1}^* y_k^*(x)}_2 \notag \\
		\le\;& L_{k+1} \qth{\norm{\pth{W_{k+1}\big| A} y_k(x) - W_{k+1} y_k^*(x)}_2 + \norm{\pth{W_{k+1}\big| A} y_k^*(x) - W_{k+1}^* y_k^*(x)}_2}\notag \\
		\le\;& L_{k+1} \qth{\norm{\pth{W_{k+1}\big| A}}_2\norm{y_k(x) - y_k^*(x)}_2  + \norm{\pth{W_{k+1}\big| A} - W_{k+1}^*}_2 \norm{y_k^*(x)}_2 } \notag\\
		\le\;& L_{k+1} \qth{\pth{\norm{W_{k+1}^*}_2 + \norm{\pth{W_{k+1}\big| A} - W_{k+1}^*}_2} \norm{y_k(x) - y_k^*(x)}_2 + \norm{\pth{W_{k+1}\big| A}- W_{k+1}^*}_2 \norm{y_k^*(x)}_2} \notag\\
		\le\;& L_{k+1} \qth{\pth{N_{k+1} + d^{-\frac{1}{4}\alpha}}\norm{{\pth{y_{k}(x)\big| A}} - y_k^*(x)}_2 + d^{-\frac{1}{4}\alpha} L_{1:k} N_{1:k}} \label{eq:23} \\
		\le\;& L_{k+1} \qth{2N_{k+1} \norm{\pth{y_{k}(x)\big| A} - y_k^*(x)}_2 + d^{-\frac{1}{4}\alpha} L_{1:k} N_{1:k}} \notag\\
		\le\;& L_{k+1} \qth{2N_{k+1} \cdot \pth{2^{k-1}-1} d^{-\frac{1}{4}\alpha} L_{1:k} N_{1:k} + d^{-\frac{1}{4}\alpha} L_{1:k} N_{1:k}} \notag\\
		\le\;& L_{k+1} \qth{2N_{k+1} \cdot \pth{2^{k-1}-1} d^{-\frac{1}{4}\alpha} L_{1:k} N_{1:k} + d^{-\frac{1}{4}\alpha} L_{1:k} N_{1:k+1}} \notag\\
		= \; & \pth{2^{k} - 1} d^{-\frac{1}{4}\alpha} L_{1:k+1} N_{1:k+1} \label{eq:induct_2_2}
	\end{align} 
	where in \eqref{eq:23} we use assumption \eqref{eq:FCN_general_d_assumption}. This finishes the induction. 
	
	We have just shown that with probability at least {$1 - (l-2) c_2 d^{-\frac{\alpha}{4}} - \sum_{i=1}^{l-1} (l-i)\delta_i$}, we have \begin{align*}
	    \norm{{\pth{y_{l-1}(x)\big| A}} - y_{l-1}^*(x)}_2 \le \pth{2^{l-2}-1} d^{-\frac{1}{4}\alpha} L_{1:l-1} N_{1:l-1}.
	\end{align*}
	
	For the last layer, by assumption, with probability at least {$\pth{1 - \delta_l} \cdot \qth{1 - (l-2) c_2 d^{-\frac{\alpha}{4}} - \sum_{i=1}^{l-1} (l-i)\delta_i}$}, we have for every $x \in \calB_{d_0}$, \begin{align*}
		\norm{{\pth{f(x)\big| A}} - F(x)}_2 &= \norm{W_{l} {\pth{y_{l-1}(x)\big| A}} - W_{l}^* y_{l-1}^*(x)}_2 \\
		& = \norm{W_{l}^* \pth{y_{l-1}(x)\big| A} - W_{l}^* y_{l-1}^*(x)}_2 \\
		& \le \norm{W_{l}^*}_2 \norm{ \pth{y_{l-1}(x)\big| A} - y_{l-1}^*(x)}_2 \\
		& \le N_{l} \pth{2^{l-2}-1} d^{-\frac{1}{4}\alpha} L_{1:l-1} N_{1:l-1} \\
		& = \pth{2^{l-2}-1} d^{-\frac{1}{4}\alpha} L_{1:l-1} N_{1:l} \\
		& \le \epsilon,
	\end{align*}
	where the last inequality follows from assumption \eqref{eq:FCN_general_d_assumption}. In conclusion, with probability at least $\prob{A} \ge \pth{1-d^{-\frac{1}{3}}}^{2(l-2)}$ over the randomness of masks, we have $\sup_{x\in \calB_{d_0}}\norm{\pth{f(x) \big| A} - F(x)}_2 \le \epsilon$ with probability at least $\pth{1 - \delta_l} \cdot \qth{1 - (l-2) c_2 d^{-\frac{\alpha}{4}} - \sum_{i=1}^{l-1} (l-i)\delta_i}$. As a result, basic probability yields that with probability at least $$p_0 \defeq \pth{1-d^{-\frac{1}{3}}}^{2(l-2)}\pth{1 - \delta_l} \qth{1 - (l-2) c_2 d^{-\frac{\alpha}{4}} - \sum_{i=1}^{l-1} (l-i)\delta_i},$$ we have $$\sup_{x\in \calB_{d_0}}\norm{f(x) - F(x)}_2 \le \epsilon.$$

	It remains to determine a lower bound of $d$ such that \begin{equation} \label{eq:requirement_d1_general}
	    d^{-\frac{1}{4}\alpha} \le \min\sth{N_2, \ldots, N_{l-1},  \frac{\epsilon}{\pth{2^{l-2}-1} L_{1:l-1} N_{1:l}}}
	\end{equation}  and \begin{equation} \label{eq:requirement_d2_general}
	    p_0 \ge 1 - \delta.
	\end{equation}
	
	
	For \eqref{eq:requirement_d1_general}, we have \begin{align} \label{eq:poly_d1_general}
	    d \ge N_k^{-\frac{4}{\alpha}}, \quad 2\le k \le l-1 \quad \textrm{ and } \quad d \ge \pth{\pth{2^{l-2}-1} L_{1:l-1} N_{1:l}}^{\frac{4}{\alpha}} \cdot \epsilon^{-\frac{4}{\alpha}}.
	\end{align}
	
	Regarding \eqref{eq:requirement_d2_general}, condition \eqref{eq:assumption_general_delta} guarantees that $ \delta_0 = \delta - \qth{\delta_l + \sum_{i=1}^{l-1}(l-i)\delta_i} \ge 0$. We have \begin{align}
	    2(l-2)d^{-\frac{1}{3}} \le \frac{2}{3}\delta_0 \quad \Leftrightarrow \quad d \ge \delta_0^{-3} \pth{3(l-2)}^3, \\
	    (l-2)c_2 d^{-\frac{\alpha}{4}} \le \frac{1}{3}\delta_0 \quad \Leftrightarrow \quad d \ge \delta_0^{-\frac{4}{\alpha}} \pth{3c_2(l-2)}^{\frac{4}{\alpha}}. \label{eq:poly_d2_general}
	\end{align}
	
	Combining \eqref{eq:poly_d1_general} - \eqref{eq:poly_d2_general}, we know that if \begin{equation*}
	    d \ge \max \sth{C_1^{\frac{4}{\alpha}}, \pth{\frac{C_2}{\epsilon}}^{\frac{4}{\alpha}}, \pth{\frac{C_3}{\delta_0}}^{3}, \pth{\frac{C_4}{\delta_0}}^{\frac{4}{\alpha}}},
	\end{equation*} for some positive constant $C_1, C_2, C_3$ and $C_4$, then with probability at least \begin{align*}
	    p_0 &= {\pth{1-d^{-\frac{1}{3}}}^{2(l-2)}\pth{1 - \delta_l} \cdot \qth{1 - (l-2) c_2 d^{-\frac{\alpha}{4}} - \sum_{i=1}^{l-1} (l-i)\delta_i}}\\
	    &\ge 1 - (l-2) c_2 d^{-\frac{\alpha}{4}} - 2(l-2) d^{-\frac{1}{3}} - \qth{\delta_l + \sum_{i=1}^{l-1} (l-i)\delta_i} \\
	    &\ge 1 - \frac{2}{3}\delta_0 - \frac{1}{3}\delta_0 - \pth{\delta - \delta_0}\\
	    &= 1-\delta,
	\end{align*} we have \begin{equation*}
	    \sup_{x\in \calB_{d_0}} \norm{f(x) - F(x)}_2 \le \epsilon.
	\end{equation*}
	
\end{proof}

\subsection{Proof of Theorem \ref{thm:FCN_uniform}}

\begin{proof}
    For any $x \in \calB_{d_0}$ and $1 \le k < l$, we denote $y_k(x) = \sigma \pth{W_k \sigma \pth{\cdots W_2 \sigma \pth{W_1 x}}}$ and $y_k^*(x) = \sigma\pth{W_k^* \sigma \pth{\cdots W_2^* \sigma \pth{W_1^* x}}}$ as the output of the $k$-th layer of $f$ and $F$, respectively. 
    
    Recall that we set $M_1$ and $M_l$ as the all 1 matrices, i.e. $W_1 = W_1^*$ and $W_l = W_l^*$. For each $1 < k < l$, we order the entries of $W_k^*$ by their absolute values such that \begin{equation*}
        \abs{(W_k^*)_{i_1^k, j_1^k}} \le \abs{(W_k^*)_{i_2^k, j_2^k}} \le \cdots \le \cdots \le \abs{(W_k^*)_{i_{D_k}^k, j_{D_k}^k}}
    \end{equation*}
    and denote $\calI_k \defeq \sth{ (i_s^k, j_s^k): 1 \le s \le \lfloor D_k^{1-\alpha} \rfloor }$. We set $(M_k)_{i,j} = 0 $ if $(i,j) \in \calI_k$, and $(M_k)_{i,j} = 1 $ otherwise. In the following, we show that $M_1, \ldots, M_l$ defined above satisfy \eqref{eq:prune_target_approximation_FCN_uniform}.
    
    By Lemma \ref{lemma:2norm}, there exist positive constants $c_0$ (depends on $K$) and $\delta_0$ such that\footnote{In fact, we get $l$ different sets of $\sth{c_i, \delta_i}, i\in [l]$ by applying Lemma \ref{lemma:2norm} $l$ times. We take $c_0 = \max \sth{c_i}$ and $\delta_0 = \min \sth{\delta_i}$ so that \eqref{eq:weight_matrix_norm} is satisfied for all $1 \le k \le l$.} \begin{align} \label{eq:weight_matrix_norm}
    	\mathbb{P}\pth{\norm{W_k^*}_2 \le c_0} \ge 1 - 2e^{-4\delta_0 d}, \quad  1\le k \le l.
    \end{align}
    
    Let us assume that \begin{equation} \label{eq:d_assumption_uniform}
        d^{-\alpha} \le \min\sth{c_0, \frac{\epsilon}{\pth{2^{l-2}-1} L_{1:(l-1)} c_0^{l-1}}}.
    \end{equation} 
    We use induction to show that, for any $x \in \calB_{d_0}$ and $1\le k < l$, \begin{enumerate}[(I)]
        \item with probability at least $\pth{1 - 2e^{-4\delta_0 d}}^{k}$, we have $\norm{y_k^*(x)}_2 \le L_{1:k} c_0^{k}$
        \item with probability at least $ 1 - (k-1)c_2 d^{-\alpha} - (k+2)(k-1) e^{-4\delta_0 d}$, we have $\norm{y_k(x) - y_k^*(x)}_2 \le \pth{2^{k-1}-1} d^{-\alpha} L_{1:k} c_0^{k-1}$.
    \end{enumerate}
    
    Statement (I) can be proved in the same way as in the proof of Theorem \ref{thm:FCN_general}. 
    We next show that (II) holds. The case of $k=1$ is trivial since $y_1(x)=y_1^*(x)$. Suppose the statement holds for $1 \le k < l-1$; we consider the case of $k+1$. Note that the non-zero entries of $\calW_{k+1} \defeq W_{k+1} - W_{k+1}^*$ are $\sth{\pth{W_{k+1}^*}_{i, j}: (i,j) \in \calI_{k+1}}$. Taking $a=\frac{K}{\sqrt{\max\sth{d_k, d_{k+1}}}}, n = D_{k+1}, r = \lfloor D_{k+1}^{1-\alpha} \rfloor$ in Lemma \ref{lemma:squaredUniform_single}, for every entry $e$ of $\calW_{k+1}$, we have \begin{align*}
        \Expect \, e^2 &\le \Expect \abs{\pth{W_{k+1}^*}_{i_{\lfloor D_{k+1}^{1-\alpha} \rfloor}^{k+1}, j_{\lfloor D_{k+1}^{1-\alpha} \rfloor}^{k+1}}}^2 = \frac{K^2}{\max\sth{d_k, d_{k+1}}} \cdot \frac{\lfloor D_{k+1}^{1-\alpha} \rfloor \pth{\lfloor D_{k+1}^{1-\alpha} \rfloor + 1}}{(D_{k+1}+1)(D_{k+1}+2)} \\
        & \le \frac{K^2}{\max\sth{d_k, d_{k+1}}} \cdot \frac{D_{k+1}^{1-\alpha}  \pth{D_{k+1}^{1-\alpha} + 1}}{(D_{k+1}+1)(D_{k+1}+2)} \le \frac{K^2}{\max\sth{d_k, d_{k+1}}} \cdot 2D_{k+1}^{-2\alpha} \le 2K^2 d^{-1-4\alpha},
        \end{align*} and similarly \begin{align*}
        \Expect \, e^4 &\le \Expect \abs{\pth{W_{k+1}^*}_{i_{\lfloor D_{k+1}^{1-\alpha} \rfloor}^{k+1}, j_{\lfloor D_{k+1}^{1-\alpha} \rfloor}^{k+1}}}^4 \le 24K^4 d^{-2-8\alpha}.
    \end{align*} 
    
    

    Taking $A = \calW_{k+1}$ in Lemma \ref{lemma:latala}, we know there exists a constant $c_2>0$ such that $ \Expect \norm{\calW_{k+1}}_2 \le c_2 d^{-2\alpha}$, where $c_2 = CK\pth{2\sqrt{2} + (24)^{1/4}}$ and $C$ is the universal constant as defined in Lemma \ref{lemma:latala}. By Markov's inequality, for all $t > 0$ we have $ \prob{\norm{\calW_{k+1}}_2 \ge t} \le \frac{\Expect\norm{\calW_{k+1}}_2}{t}$. Taking $t = d^{-\alpha}$, we have \begin{align*}
        \prob{\norm{\calW_{k+1}}_2 \le d^{-\alpha}} \ge 1 - c_2 d^{-\alpha}.
    \end{align*}
    
    Similar to \eqref{eq:induct_2_1} -- \eqref{eq:induct_2_2} in the proof of Theorem \ref{thm:FCN_general}, with probability at least {\small \begin{align}
	    & \; \mathbb{P} \bigg( \sth{\norm{W_{k+1} - W_{k+1}^*}_2 \le d^{-\alpha}} \bigcap \sth{\norm{W_{k+1}^*}_2 \le c_0} \bigcap \sth{\norm{y_k^*(x)}_2 \le L_{1:k} c_0^k} \notag \\
	    & \quad \quad \quad \quad \quad \quad \quad \quad \quad \quad \quad \quad \quad \; \; \bigcap \sth{\norm{y_k(x) - y_k^*(x)}_2 \le \pth{2^{k-1} - 1} d^{-\alpha} L_{1:k} c_0^{k-1}} \bigg) \label{eq:theorem1_events}\\
	    \ge & \; \pth{1 - c_2 d^{-\alpha}} + \pth{1-2e^{-4\delta_0 d}} + \pth{1-2e^{-4\delta_0 d}}^k + \pth{1 - (k-1)c_2 d^{-\alpha} - (k+2)(k-1)e^{-4\delta_0 d}} - 3 \notag \\
	    \ge & \; 1 - k c_2 d^{-\alpha} - (k+3)k e^{-4\delta_0 d}, \notag
	\end{align}}we have \begin{equation*}
	    \norm{y_{k+1}(x) - y_{k+1}^*(x)}_2 \le \pth{2^k - 1} d^{-\alpha} L_{1:(k+1)} c_0^k,
	\end{equation*} which finishes the induction.
	
	We have just shown that with probability at least $1 - (l-2) c_2 d^{-\alpha} - (l+1)(l-2)e^{-4\delta_0 d}$, we have \begin{align*}
	    \norm{y_{l-1}(x) - y_{l-1}^*(x)}_2 \le \pth{2^{l-2}-1} d^{-\alpha} L_{1:(l-1)} c_0^{l-2}.
	\end{align*}
	
	For the last layer, by \eqref{eq:weight_matrix_norm}, with probability at least $\pth{1 - 2e^{-4\delta_0 d}} \cdot \pth{1 - (l-2) c_2 d^{-\alpha} - (l+1)(l-2)e^{-4\delta_0 d}}$, we have for every $x \in \calB_{d_0}$, \begin{align*}
	    \norm{f(x) - F(x)}_2 \le c_0 \cdot \pth{2^{l-2}-1} d^{-\alpha} L_{1:(l-1)} c_0^{l-2} \le \epsilon,
	\end{align*}
	where the last inequality follows from assumption \eqref{eq:d_assumption_uniform}. In conclusion, we show that with probability at least $$p_0 \defeq \pth{1 - 2e^{-4\delta_0 d}} \cdot \pth{1 - (l-2) c_2 d^{-\alpha} - (l+1)(l-2)e^{-4\delta_0 d}},$$ we have $$\sup_{x\in \calB_{d_0}}\norm{f(x) - F(x)}_2 \le \epsilon.$$
	
	It remains to determine a lower bound of $d$ such that \begin{equation} \label{eq:requirement_d1}
	    d^{-\alpha} \le \min\sth{c_0, \frac{\epsilon}{\pth{2^{l-2}-1} L_{1:(l-1)} c_0^{l-1}}}
	\end{equation} and \begin{equation} \label{eq:requirement_d2}
	    p_0 \ge 1 - \delta.
	\end{equation}
	
	For \eqref{eq:requirement_d1}, we have \begin{align}
	   d \ge c_0^{-\frac{1}{\alpha}} \quad \textrm{ and } \quad d \ge \pth{\frac{\pth{2^{l-2}-1} L_{1:(l-1)} c_0^{l-1}}{\epsilon}}^{\frac{1}{\alpha}}. \label{eq:poly_d1_FCN}
	\end{align}
	
	Regarding \eqref{eq:requirement_d2}, we have $p_0 \ge 1 - (l-2) c_2 d^{-\alpha} - (l^2-l)e^{-4\delta_0 d}$. Note that $p_0 \ge 1-\delta$ if $(l-2) c_2 d^{-\alpha} \le \frac{l-2}{l^2-2}\delta$ and $(l^2-l)e^{-4\delta_0 d} \le \frac{l^2-l}{l^2-2}\delta$. These conditions are satisfied if \begin{equation}
	    d \ge \pth{\frac{(l^2-2)c_2}{\delta}}^{\frac{1}{\alpha}}	\end{equation} and \begin{equation}
	        d \ge \frac{1}{4\delta_0}\pth{\log\pth{\frac{1}{\delta}} + \log(l^2-2)}. \label{eq:poly_d4_FCN}
	    \end{equation}
	
	Combining \eqref{eq:poly_d1_FCN} - \eqref{eq:poly_d4_FCN}, we know that if \begin{equation*}
	    d \ge \max\sth{C_1^{\frac{1}{\alpha}}, \pth{\frac{C_2}{\epsilon}}^{\frac{1}{\alpha}}, \pth{\frac{C_3}{\delta}}^{\frac{1}{\alpha}}, C_4 + C_5 \log\pth{\frac{1}{\delta}}},
	\end{equation*}
	for some positive constants $C_1, C_2, C_3, C_4$ and $C_5$ specified above, then with probability at least $1-\delta$ we have  
    \begin{equation*}
	    \sup_{x\in \calB_{d_0}} \norm{f(x) - F(x)}_2 \le \epsilon.
	\end{equation*}
\end{proof}

\subsection{Proof of Theorem \ref{thm:CNN}} \label{appendix:proof3}
\begin{proof}
    Let $\calF^{(k)} \in \reals^{d_{k} \times d_{k-1} \times q_k \times q_k}$ be the corresponding convulotional tensor of $W_k^*$ and $K^{(k)} \in \reals^{d_{k} \times d_{k-1} \times p_k \times p_k}$ be as defined in \eqref{eq:calF}. For any $x \in \calC_{d_0}$ and $1 \le k < l$, we denote $y_k(x) = \sigma \pth{W_k \sigma \pth{\cdots W_2 \sigma \pth{W_1 x}}}$ and $y_k^*(x) = \sigma\pth{W_k^* \sigma \pth{\cdots W_2^* \sigma \pth{W_1^* x}}}$ as the output of the $k$-th layer of $f$ and $F$, respectively. 

    Recall that for $1<k<l$, random pruning is based on 2D filters, i.e., we randomly select $\lfloor d^{2-\alpha} \rfloor$ pairs of indices $(s', t')$ from $[d] \times [d]$ with replacement and set $\calF^{(k)}_{s',t',:,:}$ to be zero. Denote $\calI_{k} \defeq \sth{(s', t'): \calF^{(k)}_{s',t',:,:} \textrm{ is pruned}}$ and $\calM^{(k)}$ be the $d \times d$ matrix such that $\calM^{(k)}_{s',t'} = 1$ if $(s',t') \in \calI_{k}$ and $\calM^{(k)}_{s',t'} = 0$ otherwise. We further denote two events \begin{align*}
        A^{(k)}_{r} \defeq \sth{\textrm{the number of zero entries in each row of } \calM^{(k)} \textrm{ is at most } 3\lfloor d^{2-\alpha}\rfloor/d}, \\  
        A^{(k)}_{c} \defeq \sth{\textrm{the number of zero entries in each column of } \calM^{(k)} \textrm{ is at most } 3\lfloor d^{2-\alpha}\rfloor/d}
    \end{align*} and set event $A^{(k)} \defeq A^{(k)}_{r} \bigcap A^{(k)}_{c}$. Note that $\alpha \le 2 - \frac{\log(d+1)+\log^{(2)}(d)}{\log(d)}$ guarantees that $\lfloor d^{2-\alpha} \rfloor \ge d \log(d)$ and the events $A^{(k)}_{r} $ and $A^{(k)}_{c}$ are independent. Thus by Lemma \ref{lemma:ballsinbins}, we have \begin{equation*}
        \prob{A^{(k)}} = \prob{A^{(k)}_{r} \bigcap A^{(k)}_{c}} = \prob{A^{(k)}_{r}}\prob{A^{(k)}_{c}} \ge \pth{1-d^{-\frac{1}{3}}}^2.
    \end{equation*}
    Further, for $A \defeq A^{(2)} \bigcap \cdots \bigcap A^{(l-1)}$, we have $\prob{A} = \prod_{k=2}^{l-1} \prob{A^{(k)}} \ge \pth{1-d^{-\frac{1}{3}}}^{2(l-2)}$ where the probability is taken over the randomness of masks (and is not over the randomness of weights in $W_k^*$'s).
    
    For $1 \le k < l$, let $P^{(k,u,v)}\in \reals^{d\times d}, u,v\in[p]$ be as defined in Lemma \ref{lemma:sedghi} such that\footnote{Note that the dimension of $W_1^*$ and $W_l^*$ are not $p^2d \times p^2d$ and thus we cannot apply Lemma \ref{lemma:sedghi} directly. However, we can always embed them into a $p^2d \times p^2d$ matrix. For example, we can define $\tilde{W}_l^* = [W_l^*, \vec{0}_{p^2d\times p^2(d-d_l)}]$ and apply Lemma \ref{lemma:sedghi} on $\tilde{W}_l^*$. We use the fact that $\norm{W_l^*}_2 \le \norm{\tilde{W}_l^*}_2$ to get the same result.} \begin{equation*}
        \norm{W_k^*}_2 = \max_{u,v\in[p]} \sth{\norm{P^{(k,u,v)}}_2}.
    \end{equation*}
    Recall that $\omega = \exp\pth{2\pi \sqrt{-1} / p}$ and $S \in \reals^{p\times p}$ is the matrix of the discrete Fourier transform. By Lemma \ref{lemma:sedghi}, the $(s,t)$-th entry of $P^{(k,u,v)}$ can be written as \begin{equation*}
        P^{(k, u, v)}_{s,t} = \pth{S^T K_{s,t,:,:} S}_{u,v} = \sum_{i, j \in [p]} \omega^{ui} K^{(k)}_{s,t,i,j} \omega^{vj} = \sum_{i, j \in [q]} \omega^{ui} K^{(k)}_{s,t,i,j} \omega^{vj}, \quad s, t \in [d], u, v \in [p],
    \end{equation*} where the last equality is due to \eqref{eq:calF} since $K^{(k)}_{s,t,:,:}$ has non-zero entries only in its top-left $q\times q$ sub-matrix.
    
    Denoting $P^{(k, u,v,i,j)} \defeq \omega^{ui+vj} K^{(k)}_{:,:,i,j}, u,v \in [p], i,j\in[q]$, then we have $P^{(k, u, v)} \defeq \sum_{i,j \in [q]} P^{(k, u,v,i,j)}$ and \begin{equation*}
        \norm{P^{(k, u,v,i,j)}}_2 = \norm{\omega^{ui+vj} K^{(k)}_{:,:,i,j}}_2 = \norm{K^{(k)}_{:,:,i,j}}_2 = \norm{\calF^{(k)}_{:,:,i,j}}_2, \quad u,v\in[p], i,j\in[q].
    \end{equation*}
    
    By assumption (iii), $\calF^{(k)}_{:,:,i,j} \in \reals^{d\times d}$ is a random matrix whose entries are independently sampled from different distributions. In addition, these distributions' second-order moments are upper-bounded by $\frac{C_1}{p^2 d}$ and the fourth-order moments are upper-bounded by $\frac{C_2}{p^4d^2}$. By Lemma \ref{lemma:latala}, for all $i,j \in [q]$, there exists a universal constant $C > 0$ such that \begin{equation} \label{eq:filter_norm_CNN}
    \Expect \norm{\calF^{(k)}_{i,j,:,:}}_2 \le C \qth{\pth{d \frac{C_1}{p^2 d}}^{\frac{1}{2}} + \pth{d \frac{C_1}{p^2 d}}^{\frac{1}{2}} + \pth{d^2 \frac{C_2}{p^4 d^2}}^{\frac{1}{4}}} \le \frac{C_3}{p},
    \end{equation} where $C_3 = C\pth{2\sqrt{C_1}+C_2^{1/4}}$.
    
    Thus we have \begin{equation}\label{eq:filter_norm_CNN1}
        \Expect \norm{W_k^*}_2 \le \max_{u,v\in[p]} \sth{\Expect \norm{P^{(k,u,v)}}_2} \le \max_{u,v\in[p]} \sth{\sum_{i,j\in[q]}\Expect \norm{P^{(k,u,v,i,j)}}_2} = \max_{u,v\in[p]} \sth{\sum_{i,j\in[q]}\Expect \norm{\calF^{(k)}_{:,:,i,j}}_2} \le C_3\frac{q^2}{p}.
    \end{equation}
    
    By the Markov's inequality, we have \begin{equation} \label{eq:layer_norm_CNN}
        \prob{\norm{W_k^*}_2 \le p^{-\beta_1}} \ge 1- C_3 \frac{q^2}{p^{1-\beta_1}}.
    \end{equation}
    
    We use induction to show that, for any $x \in \calC_{p_0^2 d_0}$ and $1\le k < l$, we have \begin{enumerate}[(I)]
        \item with probability at least $\pth{1 - C_3 \frac{q^2}{p^{1-\beta_1}}}^{k}$, we have $\norm{y_k^*(x)}_2 \le \pth{L p^{-\beta_1}}^k p_0 \sqrt{d_0}$,
        \item with probability at least $ 1-(k-1)C_4 \frac{q^2}{p}d^{-\frac{1}{4}\alpha + \beta_2} - \frac{k^2+k-2}{2}C_3\frac{q^2}{p^{1-\beta_1}}$, we have $ \norm{\pth{y_k(x) \big| A} - y_k^*(x)}_2 \le \pth{p^{-\beta_1}\pth{p^{-\beta_1} + d^{-\beta_2}}^{k-1}-p^{-k\beta_1}} L^k p_0\sqrt{d_0}$ holds for some positive constant $C_4$ specified later\footnote{Note that in induction statement (II), the probability (and the expectations in the following context) is taken over the randomness of weights but not the masks, the random variable $\norm{\pth{y_k(x) \big| A} - y_k^*(x)}_2$ is equivalent to $\norm{y_k(x) - y_k^*(x)}_2\big| A$. Further, the statement can also be written as \begin{align*}
            \prob{\sth{\norm{y_k(x) - y_k^*(x)}_2\le \pth{p^{-\beta_1}\pth{p^{-\beta_1} + d^{-\beta_2}}^{k-1}-p^{-k\beta_1}} L^k p_0\sqrt{d_0}} \Big| A} \ge 1-(k-1)C_4 \frac{q^2}{p}d^{-\frac{1}{4}\alpha + \beta_2} - \frac{k^2+k-2}{2}C_3\frac{q^2}{p^{1-\beta_1}}.
        \end{align*}}.
    \end{enumerate}
    
    The case of $k=1$ is as follows. With probability at least $1 - C_3 \frac{q^2}{p^{1-\beta_1}}$, we have $\norm{y_1^*(x)}_2 =\norm{\sigma\pth{W_1^*x}}_2 \le L \norm{W_1^*x}_2 \le L \norm{W_1^*}_2 \norm{x}_2 \le L p^{-\beta_1} p_0 \sqrt{d_0} $. Further, we have $y_1(x) = \sigma\pth{W_1x} = \sigma\pth{W_1^*x} = y_1^*(x)$, and thus $\norm{y_1(x) - y_1^*(x)}_2 = 0$.
    
    Suppose the statement holds for $1 \le k < l-1 $, we consider the case of $k+1$. Note that the events $\sth{\norm{W_{k+1}^*}_2 \le p^{-\beta_1}}$ and $\sth{\norm{y_k^*(x)}_2 \le \pth{L p^{-\beta_1}}^k p_0 \sqrt{d_0}}$ are independent. By \eqref{eq:layer_norm_CNN} and the induction statement (I), with probability at least $$\prob{\norm{W_{k+1}^*}_2 \le p^{-\beta_1}}\prob{\norm{y_{k}^*(x)}_2 \le \pth{L p^{-\beta_1}}^k p_0 \sqrt{d_0}} \ge \pth{1 - C_3 \frac{q^2}{p^{1-\beta_1}}}^{k+1},$$ we have  \begin{align*}
    	\norm{y_{k+1}^*(x)}_2 &= \norm{\sigma \pth{W_{k+1}^* y_k^*(x)}}_2 \le L \norm{W_{k+1}^* y_k^*(x)}_2 \le L \norm{W_{k+1}^* }_2 \norm{y_k^*(x)}_2 \\
    	&\le L p^{-\beta_1} \cdot \pth{L p^{-\beta_1}}^k p_0 \sqrt{d_0} = \pth{L p^{-\beta_1}}^{k+1} p_0 \sqrt{d_0},
    \end{align*} which shows (I) in the induction statement. 
    
    We use a similar approach as in the proof for Theorem \ref{thm:FCN_general} to show that (II) holds. Let us denote $\widebar{K}^{(k+1)}_{:,:i,j} \defeq \calM^{(k+1)} \circ K^{(k+1)}_{:,:,i,j}, i, j \in [p] $, i.e., $\widebar{K}^{(k+1)}_{s',t',:,:} =  K^{(k+1)}_{s',t',:,:}$ if $(s',t') \in \calI_{k+1}$ and $\widebar{K}^{(k+1)}_{s',t',:,:} = \vec{0}_{p\times p}$ otherwise. Then $\widebar{W}_{k+1} \defeq W_{k+1}^* - W_{k+1} $ can be represented by \begin{equation}  
    \widebar{W}_{k+1} = \begin{bmatrix}
        \widebar{B}^{(k+1)}_{1,1} & \cdots & \widebar{B}^{(k+1)}_{1, d} \\
        \vdots & \ddots & \vdots \\
        \widebar{B}^{(k+1)}_{d', 1} & \cdots & \widebar{B}^{(k+1)}_{d', d}
    \end{bmatrix},
    \end{equation} where each $\widebar{B}^{(k+1)}_{s,t}$ is a doubly block circulant matrix such that \begin{equation}
        \widebar{B}^{(k+1)}_{s,t} = \begin{bmatrix}
            \textrm{circ}\pth{\widebar{K}^{(k+1)}_{s,t,1,:}} & \textrm{circ}\pth{\widebar{K}^{(k+1)}_{s,t,2,:}} & \cdots & \textrm{circ}\pth{\widebar{K}^{(k+1)}_{s,t,p,:}}\\
            \textrm{circ}\pth{\widebar{K}^{(k+1)}_{s,t,p,:}} & \textrm{circ}\pth{\widebar{K}^{(k+1)}_{s,t,1,:}} & \cdots & \textrm{circ}\pth{\widebar{K}^{(k+1)}_{s,t,p-1,:}} \\
            \vdots & \vdots & \ddots & \vdots \\
            \textrm{circ}\pth{\widebar{K}^{(k+1)}_{s,t,2,:}} & \textrm{circ}\pth{\widebar{K}^{(k+1)}_{s,t,3,:}} & \cdots &  \textrm{circ}\pth{\widebar{K}^{(k+1)}_{s,t,1,:}}
        \end{bmatrix}.
    \end{equation}
    
    Again, let $\widebar{P}^{(k+1,u,v)}\in \reals^{d\times d}, u,v\in[p]$ be such that \begin{equation*}
        \widebar{P}^{(k+1,u,v)}_{s,t} = \pth{S^T \widebar{K}^{(k+1)}_{s,t,:,:}S}_{u,v} = \sum_{i, j \in [q]} \omega^{ui} \widebar{K}^{(k+1)}_{s,t,i,j} \omega^{vj}, \quad s, t \in [d], u, v \in [p],
    \end{equation*} and we denote $\widebar{P}^{(k+1,u,v,i,j)} \defeq \omega^{ui+vj} \widebar{K}^{(k+1)}_{:,:,i,j}, u,v\in[p],i,j\in[q]$. Then we have $\widebar{P}^{(k+1,u,v)} = \sum_{i,j\in[q]} \widebar{P}^{(k+1,u,v,i,j)}$ and \begin{equation*}
            \norm{\widebar{P}^{(k+1, u,v,i,j)}}_2 = \norm{\omega^{ui+vj} \widebar{K}^{(k+1)}_{:,:,i,j}}_2 = \norm{\widebar{K}^{(k+1)}_{:,:,i,j}}, \quad u,v\in[p], i,j\in[q].
        \end{equation*}
        
    By assumption (iii), every entry of $\widebar{K}^{(k+1)}_{:,:,i,j}$ follows a distribution such that the second-order moment is upper-bounded by $\frac{C_1}{p^2d}$ and the fourth-order moment is upper-bounded by $\frac{C_2}{p^4d^2}$. Under event $A$, that the number of non-zero entries in $\widebar{K}^{(k+1)}_{:,:,i,j}$ is at most $\lfloor d^{2-\alpha} \rfloor$ and the number of non-zero entries in each row/column of $\widebar{K}^{(k+1)}_{:,:,i,j}$ is at most $3\lfloor d^{2-\alpha} \rfloor/d$, by Lemma \ref{lemma:latala} and a similar derivation to \eqref{eq:latala1} -- \eqref{eq:apply_latala}, we have \begin{equation*}
        \condexp{}{\norm{\widebar{K}^{(k+1)}_{:,:,i,j}}_2}{A}  \le \frac{C_4}{p}d^{-\frac{1}{4}\alpha},
    \end{equation*}
    where $C_4 = C\pth{2\pth{3C_1}^{\frac{1}{2}} + C_2^{\frac{1}{4}}}$ and $C$ is the universal constant as defined in Lemma \ref{lemma:latala}.
    
    By Lemma \ref{lemma:sedghi}, we have \begin{align*}
        \condexp{}{\norm{W_{k+1}^* - W_{k+1}}_2}{A} &= \condexp{}{\norm{\widebar{W}_{k+1}}_2}{A} = \max_{u,v\in[p]} \sth{\condexp{}{\norm{\widebar{P}^{(k+1,u,v)}}_2}{A}} \\
        &\le \max_{u,v\in[p]} \sth{\sum_{i,j\in[q]} \condexp{}{\norm{\widebar{P}^{(k+1,u,v,i,j)}}_2}{A}} \\
        &= \max_{u,v\in[p]} \sth{\sum_{i,j\in[q]}\condexp{}{\norm{\widebar{K}^{(k+1)}_{:,:,i,j}}_2}{A}}\\
        &\le C_4 \frac{q^2}{p}d^{-\frac{1}{4}\alpha}.
    \end{align*}
    
    
    By the Markov's inequality, for all $t > 0$ we have \begin{align*}
        \prob{\sth{\norm{W_{k+1}^* - W_{k+1}}_2 \ge t}\Big| A} \le \frac{\condexp{}{\norm{W_{k+1}^* - W_{k+1}}_2}{A}}{t}.
    \end{align*}
    Taking $t = d^{-\beta_2}$, we have \begin{align*}
        \prob{\sth{\norm{W_{k+1}^* - W_{k+1}}_2 \le d^{-\beta_2}} \Big| A} \ge 1 - C_4 \frac{q^2}{p} d^{-\frac{1}{4}\alpha+\beta_2}.
    \end{align*}
    
    Similar to \eqref{eq:induct_2_1} -- \eqref{eq:induct_2_2} in the proof of Theorem \ref{thm:FCN_general}, with probability at least 
    \begin{align}
	    & \; \mathbb{P}\bigg(\sth{\norm{\pth{W_{k+1}\big| A} - W_{k+1}^*}_2 \le d^{-\beta_2}} \bigcap \sth{\norm{W_{k+1}^*}_2 \le p^{-\beta_1}} \bigcap \sth{\norm{y_k^*(x)}_2 \le \pth{Lp^{-\beta_1}}^{k} p_0 \sqrt{d} } \notag\\
	    & \qquad\qquad\qquad\quad \bigcap \sth{\norm{\pth{y_k(x)\big| A} - y_k^*(x)}_2 \le \pth{p^{-\beta_1}\pth{p^{-\beta_1} + d^{-\beta_2}}^{k-1}-p^{-k\beta_1}}L^k p_0\sqrt{d_0}} \bigg) \label{eq:theorem2_events}\\
	    \ge & \; \pth{1-C_4 \frac{q^2}{p}d^{-\frac{1}{4}\alpha + \beta_2}} + \pth{1 - C_3\frac{q^2}{p^{1-\beta_1}}} + \pth{1 - C_3\frac{q^2}{p^{1-\beta_1}}}^k + \pth{1-(k-1)C_4 \frac{q^2}{p}d^{-\frac{1}{4}\alpha + \beta_2} - \frac{k^2+k-2}{2}C_3\frac{q^2}{p^{1-\beta_1}}} - 3\notag\\
	    = & \; 1 - kC_4 \frac{q^2}{p}d^{-\frac{1}{4}\alpha + \beta_2} - \frac{(k+1)^2+(k+1)-2}{2}C_3\frac{q^2}{p^{1-\beta_1}},\notag
	\end{align}
    
    
     we have \begin{align*}
	    \norm{\pth{y_{k+1}(x)\big| A} - y_{k+1}^*(x)}_2 \le L^{k+1} p_0 \sqrt{d} \qth{ p^{-\beta_1} \pth{p^{-\beta_1} + d^{-\beta_2}}^{k} - p^{-(k+1)\beta_1}},
	\end{align*} which finishes the induction.
    
    
	We have just shown that with probability at least $1 - (l-2)C_4 \frac{q^2}{p}d^{-\frac{1}{4}\alpha + \beta_2} - \frac{l^2-l-2}{2}C_3\frac{q^2}{p^{1-\beta_1}}$, we have $$\norm{\pth{y_{l-1}(x)\big| A} - y_{l-1}^*(x)}_2 \le L^{l-1} p_0 \sqrt{d} \qth{ p^{-\beta_1} \pth{p^{-\beta_1} + d^{-\beta_2}}^{l-2} - p^{-(l-1)\beta_1}}.$$
	
	Note that the last layer of $F$ is a fully-connected layer with dimension $p^2d \times d_l$. By Lemma \ref{lemma:latala}, the Markov's inequality, and a similar derivation to \eqref{eq:filter_norm_CNN} -- \eqref{eq:filter_norm_CNN1}, there exists a positive constant $C_5$ such that $\prob{\norm{W_l^*} \le p^{-\beta_1}} \ge 1 - \frac{C_5}{p^{1-\beta_1}}$. Therefore, with probability at least \begin{align*}
	    &\pth{1 - \frac{C_5}{p^{1-\beta_1}}} \cdot \pth{1 - (l-2)C_4 \frac{q^2}{p}d^{-\frac{1}{4}\alpha + \beta_2} - \frac{l^2-l-2}{2}C_3\frac{q^2}{p^{1-\beta_1}}} \\
	    \ge \; & 1 - (l-2)C_4 \frac{q^2}{p}d^{-\frac{1}{4}\alpha + \beta_2} - \frac{l^2-l-2}{2}C_3\frac{q^2}{p^{1-\beta_1}} - \frac{C_5}{p^{1-\beta_1}},
	\end{align*} we have that for every $x \in \calC_{p_0^2d_0}$ \begin{align*}
		\norm{\pth{f(x)\big| A} - F(x)}_2 \le p^{-\beta_1} L^{l-1} p_0 \sqrt{d} \qth{ p^{-\beta_1} \pth{p^{-\beta_1} + d^{-\beta_2}}^{l-2} - p^{-(l-1)\beta_1}}.
	\end{align*}
	
    With probability at least $\prob{A} = \pth{1-d^{-\frac{1}{3}}}^{2(l-2)}$ over the randomness of masks, we have \begin{equation*}
        \sup_{x \in \calC_{p_0^2d_0}} \norm{\pth{f(x) \big| A} - F(x)}_2 \le p^{-\beta_1} L^{l-1} p_0 \sqrt{d} \qth{ p^{-\beta_1} \pth{p^{-\beta_1} + d^{-\beta_2}}^{l-2} - p^{-(l-1)\beta_1}}
    \end{equation*} with probability at least $1 - (l-2)C_4 \frac{q^2}{p}d^{-\frac{1}{4}\alpha + \beta_2} - \frac{l^2-l-2}{2}C_3\frac{q^2}{p^{1-\beta_1}} - \frac{C_5}{p^{1-\beta_1}}$. As a result, basic probability yields that \begin{equation*}
        \sup_{x \in \calC_{p_0^2d_0}} \norm{f(x) - F(x)}_2 \le p^{-\beta_1} L^{l-1} p_0 \sqrt{d} \qth{ p^{-\beta_1} \pth{p^{-\beta_1} + d^{-\beta_2}}^{l-2} - p^{-(l-1)\beta_1}}
    \end{equation*} holds with probability at least $\pth{1-d^{-\frac{1}{3}}}^{2(l-2)}\pth{1 - (l-2)C_4 \frac{q^2}{p}d^{-\frac{1}{4}\alpha + \beta_2} - \frac{l^2-l-2}{2}C_3\frac{q^2}{p^{1-\beta_1}} - \frac{C_5}{p^{1-\beta_1}}}$.

\end{proof}

\section{Extension of Magnitude-based Pruning} \label{appendix:sec:magnitude}

In this section, we discuss some extensions of Theorems \ref{thm:FCN_uniform} and \ref{thm:CNN} presented in the main paper. Note that we only provide ideas but not strict proofs in this section, as the results here are based on approximations and further efforts are required to give precise statements.

\subsection{Magnitude-based Pruning of FCNs with Sub-Gaussian Distributions} \label{appendix:magnitude:FCN}

Note that in Theorem \ref{thm:FCN_uniform}, assumption (iii), we assume that the distribution of the weights in the layers of $F$ are  independently and identically following $\calU\qth{-\frac{K}{\sqrt{\max\sth{d_k, d_k-1}}}, \frac{K}{\sqrt{\max\sth{d_k, d_k-1}}}}$. The uniform distribution provides a closed-form order statistics and hence we can bound the gap between weight matrices and pruned weight matrices precisely.
In fact, the uniform and exponential distributions are the only distributions that have a closed-form for order statistics in the literature. 
It is a natural question of what happens if the weights follow a more general distribution, e.g. a sub-Gaussian distribution.

Consider a target weight matrix $W^* \in \reals^{d\times d}$ where we prune the smallest $\lfloor d^{2-\alpha} \rfloor$ entries in $W^*$ based on magnitude. We further assume that the weights in $W^*$ independently and identically follow a sub-Gaussian distribution $\textsf{subG}(\sigma^2)$ with appropriate choice of $\sigma^2$ (e.g., $\sigma^2 = \frac{1}{d}$). Next we present the idea of applying the results of intermediate order statistics to show a similar result in the asymptotic sense.

\begin{theorem}[Lemma 1 of \citet{chibisov1964limit}] \label{theorem:intermediate}
    Let $X_1, X_2, \ldots$ be a sequence of independent random variables with the same distribution function $F$. We denote $X_m^{(n)}$ as the $m$-th largest among $X_1, \ldots, X_n$ and $G_{m,n}(x) = \prob{X_{m}^{(n)} < x}$. If $n \goto \infty, m \goto \infty$, and $m/n \goto 0$, then \begin{equation*}
        \sup_{x} \abs{G_{m,n}(a_n x + b_n) - \Phi\pth{u_n(x)}} \goto 0,
    \end{equation*} where $u_n(x) = \frac{nF(a_n x + b_n)}{\sqrt{m}}$ and $\Phi$ is the cumulative distribution function of the standard Gaussian distribution.
\end{theorem}

Note that the non-zero entries of $\calW \defeq W^* - W$ are the smallest $\lfloor d^{2-\alpha} \rfloor$ order statistics of $\textsf{subG}(\sigma^2)$ based on magnitude, where $W$ is the pruned weight matrix. If we order the weights in $W^*$ by their magnitude, i.e. \begin{equation*}
    \abs{W^*_{i_1, j_1}} \le \abs{W^*_{i_2, j_2}} \le \cdots \le \abs{W^*_{i_{d^2}, j_{d^2}}},
\end{equation*}
then the non-zero entries in $\calW$ are $$ W^*_{i_1, j_1}, W^*_{i_2, j_2}, \ldots,  W^*_{i_{ \lfloor d^{2-\alpha} \rfloor}, j_{ \lfloor d^{2-\alpha} \rfloor}}.$$

Taking $m = \lfloor d^{2-\alpha} \rfloor, n = d^2, a_n = 1, b_n = 0$ in Theorem \ref{theorem:intermediate} and note that $G_{m,n}(x) = \prob{\abs{W^*_{i_{ \lfloor d^{2-\alpha} \rfloor}, j_{ \lfloor d^{2-\alpha} \rfloor}}} \le x}$, we have \begin{equation*}
     \sup_x \abs{\prob{\abs{W^*_{i_{ \lfloor d^{2-\alpha} \rfloor}, j_{ \lfloor d^{2-\alpha} \rfloor}}} \le x} - \Phi\pth{\frac{d^2 F(x)}{\sqrt{\lfloor d^{2-\alpha} \rfloor}}}} \goto 0, \quad x \goto \infty.
\end{equation*}
Thus we can approximate the expectation $\Expect \qth{\abs{W^*_{i_{ \lfloor d^{2-\alpha} \rfloor}, j_{ \lfloor d^{2-\alpha} \rfloor}}}}$ by some positive constant $\beta$, by the properties of the cumulative density function of standard Gaussian and $\textsf{subG}\pth{\sigma^2}$.
Similarly, we can get the estimations of $\Expect \qth{\abs{W^*_{i_{ \lfloor d^{2-\alpha} \rfloor}, j_{ \lfloor d^{2-\alpha} \rfloor}}}^2}$ and $\Expect \qth{\abs{W^*_{i_{ \lfloor d^{2-\alpha} \rfloor}, j_{ \lfloor d^{2-\alpha} \rfloor}}}^4}$. Then we can apply Lemma \ref{lemma:latala} (similar to \eqref{eq:latala1} -- \eqref{eq:apply_latala}) to upper-bound the expectation $\Expect \norm{\calW}_2$. Recall that this is an asymptotic derivation, and we also need to bound the gap between the above second and fourth-order moments when $n = d^2$ is a large but fixed.


\subsection{Magnitude-based Pruning of CNNs} \label{appendix:magnitude:CNN}

We are given a convolutional tensor $\calF\in \reals^{d\times d \times p \times p} $. Let $$W^* = \begin{bmatrix}
    B_{1,1} & \cdots & B_{1, d} \\
    \vdots & \ddots & \vdots \\
    B_{d, 1} & \cdots & B_{d, d}
\end{bmatrix} \in \reals^{p^2d \times p^2d}$$ be the linear transformation corresponding to $\calF$, and tensor $K$ and $B$ as defined in \eqref{eq:calF} -- \eqref{eq:doubly_block}. The magnitude-based filter pruning of CNN is to order the $L_1$ norms $\norm{\textrm{vec}\pth{B_{i,j}}}_1, i,j\in[d]$ (or equivalently, $\norm{\textrm{vec}\pth{K_{i,j,:,:}}}_1$) and set the filters with the smallest $L_1$ norms to be zero. In other words, if we denote $W$ to be the pruned weight matrix, then $$W_* - W = \begin{bmatrix}
    \widebar{B}_{1,1} & \cdots & \widebar{B}_{1, d} \\
    \vdots & \ddots & \vdots \\
    \widebar{B}_{d, 1} & \cdots & \widebar{B}_{d, d}
\end{bmatrix}$$ is a block matrix of $\widebar{B}_{i,j}$, where $\widebar{B}_{i,j} = B_{i,j}$ if $\norm{\textrm{vec}\pth{B_{i,j}}}_1$ is among the smallest $\lfloor d^{2-\alpha} \rfloor$ norms, and $\widebar{B}_{i,j} = \vec{0}_{p\times p}$ otherwise. Similar to Appendix \ref{appendix:magnitude:FCN}, we can upper-bound $\Expect\norm{W^* - W}_2$ by $\Expect \norm{\widebar{B}_0}_2^2$ and $\Expect \norm{\widebar{B}_0}_2^4$, where $\widebar{B}_0 \in \sth{\widebar{B}_{i,j}: i,j\in [d]}$ is the matrix corresponding ot the $\lfloor d^{2-\alpha} \rfloor$-th smallest value based on $L_1$ norms.

Note that the $L_1$ norms are the sum of many random samples drawn from a given distribution. By the Central Limit Theory, $\norm{\textrm{vec}\pth{B_{i,j}}}_1$ can be approximated by a normal distribution. Thus we can estimate $\norm{\textrm{vec}\pth{\widebar{B}_0}}_1$ by a similar approach to the one in Appendix \ref{appendix:magnitude:FCN}. Theorem 6 of \citet{sedghi2018singular} further provides a tool to upper-bound $\norm{\widebar{B}_0}_2$ by $\norm{\textrm{vec}\pth{\widebar{B}_0}}_1$.

Note that we use two approximations in the above derivation. One is for the distribution of $\norm{\textrm{vec}\pth{B_{i,j}}}_1, i,j\in[d]$ and the other one comes from the asymptotic result as discussed in Appendix \ref{appendix:magnitude:FCN}. Caution should be taken while following these steps to attack the magnitude-based pruning problem of CNNs.

\section{Numerical Study} \label{appendix:sec:numerical}

In Sections \ref{appendix:subsec:distribution} and \ref{appendix:subsec:distribution_vgg16}, we show the histograms of some trained FCNs and CNNs. In Sections \ref{appendix:subsec:uniform} and \ref{appendix:subsec:latala}, we show the universal constants in Lemmas \ref{lemma:2norm} and \ref{lemma:latala} as we use them frequently in the paper.

\subsection{Distribution of Weights in Trained FCNs} \label{appendix:subsec:distribution}

We first describe the setting where we train a vanilla FCN. The \texttt{Covertype} dataset \citep{blackard1998comparative} is to predict 7 different forest cover types from cartographic variables. Data is in raw form (not scaled) and contains binary (0 or 1) columns of data for qualitative independent variables (wilderness areas and soil types). The dataset contains about 580,000 samples with 9 numerical and 44 categorical features. We normalize the numerical features by mean and variance of each feature. We build a 5-hidden-layer fully-connected neural network with ReLU activation functions to predict the label of each sample. {There are 1,024 neurons in each hidden layer and thus the first weight matrix has dimension $54\times 1024$, the internal 4 weight matrices have dimension $1024 \times 1024$, and the last weight matrix has dimension $1024 \times 7$.} We minimize the cross-entropy loss using Adam with learning rate 0.001. The batch-size is selected to be 512 and we run 20 epochs of training. The trained network achieves approximately 80\% predicting accuracy.

Figure \ref{fig:hist} in the main paper shows the histogram of the entries in all weight matrices. We mainly focus on the second to fifth layers because we do not perform any pruning on the first and last layers. In these 4 layers, the weights are approximately distributed following a Gaussian distribution. We also report the means and variances of the entries in each internal layer in Table \ref{table:hist}. As we can see from the results, for the internal weight matrices, the means are close to zero while the variances are approximately bounded by $\frac{3}{1024}$, which is also the initialization variance suggested by \citet{glorot2010understanding}. {We have also tested several other random initial weights and network architectures, and the results and conclusions are similar and not presented.}

\begin{table}[h!]
    \caption{Expectation and variance of the entries in all weight matrices}
    \label{table:hist}
    \centering
    \begin{tabular}{ccccccc}
    \toprule
		Layer  & 1 & 2 & 3 & 4 & 5 & 6 \\
    \midrule
        Mean  & -0.0309 & -0.0215 & -0.0078 & -0.0119 & -0.0092 & -0.0275  \\
        Variance & 0.0155 & 0.0035 & 0.0022 & 0.0019 & 0.0016 & 0.0057  \\
        \bottomrule
    \end{tabular}
\end{table}

\subsection{Distribution of Weights in VGG16}\label{appendix:subsec:distribution_vgg16}

We plot the histogram of weights in different layers of VGG16 \citep{simonyan2014very}. The pre-trained model is imported from \texttt{PyTorch} package \cite{NEURIPS2019_9015} {where the weights are trained on a variety of image datasets}. Figure \ref{fig:VGG16} shows the results for all layers of the pre-trained VGG16 (13 convulotional layers and 3 fully-connected layers). As we can see, the entries in the internal layers follow Gaussian distributions approximately.

\begin{figure}[h!]
    \centering
    \subfigure{\includegraphics[width=0.24\textwidth]{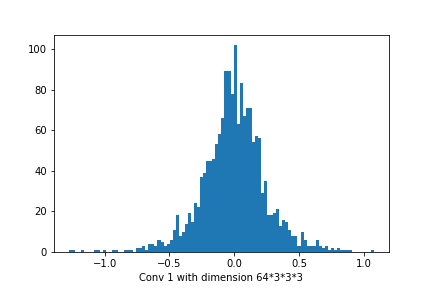}} 
    \subfigure{\includegraphics[width=0.24\textwidth]{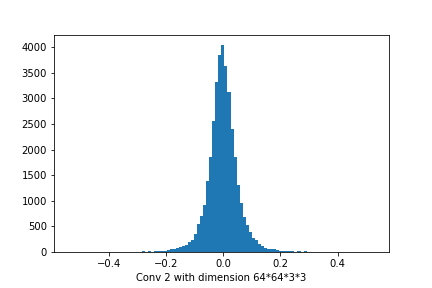}}
    \subfigure{\includegraphics[width=0.24\textwidth]{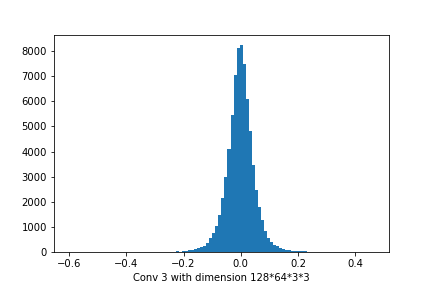}}
    \subfigure{\includegraphics[width=0.24\textwidth]{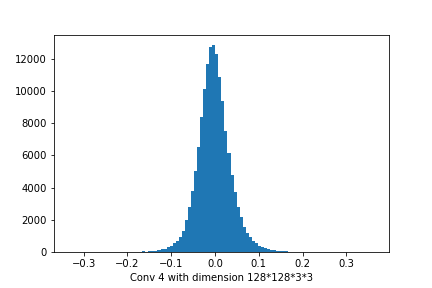}}\\
    \subfigure{\includegraphics[width=0.24\textwidth]{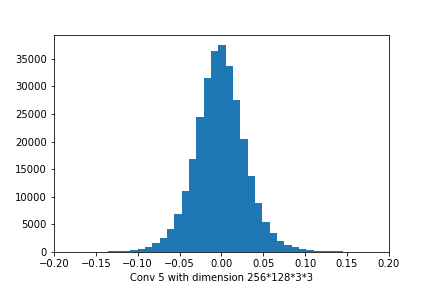}} 
    \subfigure{\includegraphics[width=0.24\textwidth]{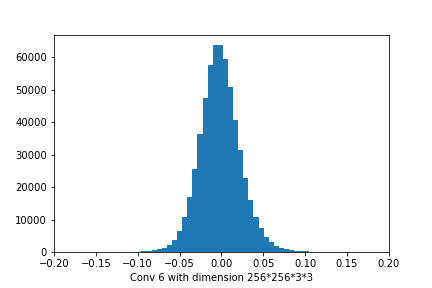}}
    \subfigure{\includegraphics[width=0.24\textwidth]{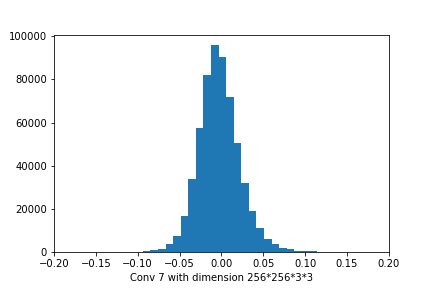}}
    \subfigure{\includegraphics[width=0.24\textwidth]{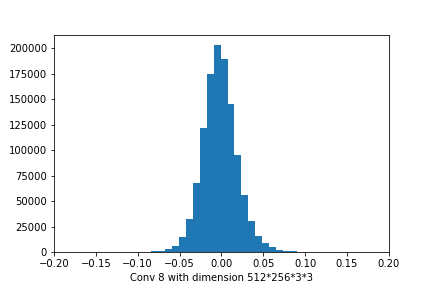}}\\
    \subfigure{\includegraphics[width=0.24\textwidth]{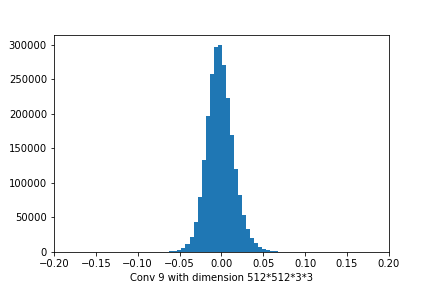}} 
    \subfigure{\includegraphics[width=0.24\textwidth]{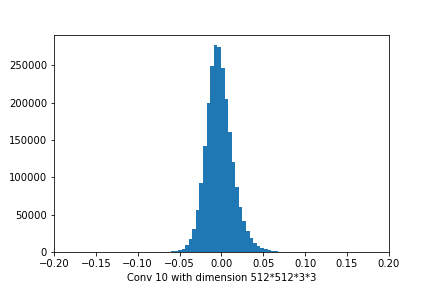}}
    \subfigure{\includegraphics[width=0.24\textwidth]{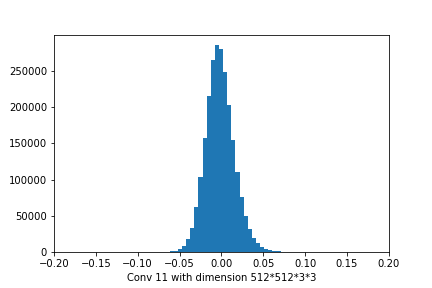}}
    \subfigure{\includegraphics[width=0.24\textwidth]{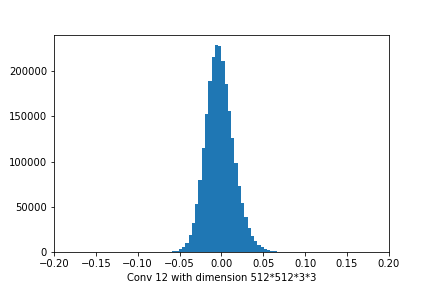}}\\
    \subfigure{\includegraphics[width=0.24\textwidth]{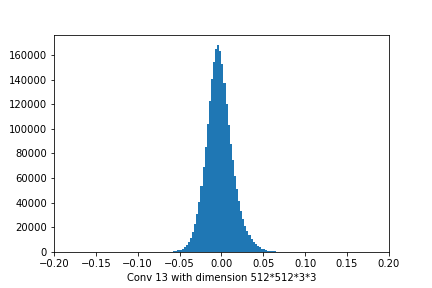}} 
    \subfigure{\includegraphics[width=0.24\textwidth]{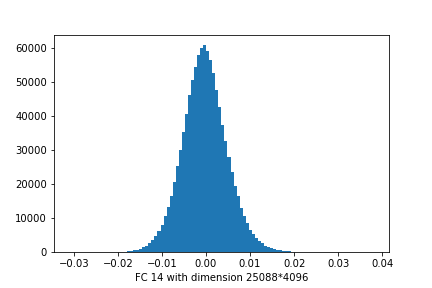}}
    \subfigure{\includegraphics[width=0.24\textwidth]{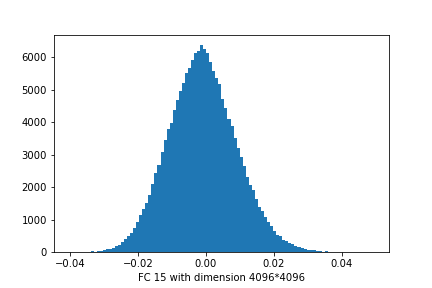}}
    \subfigure{\includegraphics[width=0.24\textwidth]{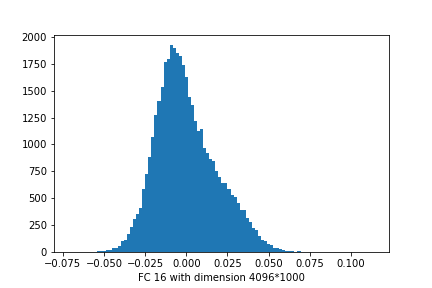}}\\
    \caption{Histogram of entries of all weight matrices of a pre-trained VGG16}
    \label{fig:VGG16}
\end{figure}

\subsection{Constants in Lemma \ref{lemma:2norm}} \label{appendix:subsec:uniform}

Lemma \ref{lemma:2norm} gives an upper-bound of the random matrix $B \in \reals^{n_1 \times n_2}$ whose entries are independently and identically following a uniform distribution $\calU\qth{-\frac{K}{\sqrt{n}}, \frac{K}{\sqrt{n}}}$, where $n = \max\sth{n_1, n_2}$ and $K$ is a positive constant. To better understand the values of constants $c_0$ and $\delta_0$, we take various tuples of $(n_1, n_2, K)$ and calculate the norm $\norm{B}_2$. In the numerical experiments, we generate in total $N=1000$ random matrices and report $c_0$ and $\delta_0$ that satisfy $\prob{\norm{B}_2 \le c_0} = 1 - 2e^{-4\delta_0 n} = q$ for $q = 95\%, 99\%, 99.9\%, 99.99\%$. We also report the mean and standard deviation of $\norm{B}_2$ for reference. The results are given in Table \ref{table:1}. {The table shows that, even if $n_1$ and $n_2$ are on the low end with respect to the actual use cases, we can still have a small $c_0$ that is close to 1 and a small $\delta_0$ that is close to 0. Note that these two quantities are frequently used in Theorem \ref{thm:FCN_uniform} and we observe that the constant terms in the theorem are mild while the probability that the statement hold is positive.}
\begin{table}[h!]
    \caption{Numerical results for the constants in Lemma \ref{lemma:2norm}}
    \label{table:1}
    \centering
    \begin{tabular}{ccccccccccccc}
    \toprule
		\multirow{2}{*}{$n_1$} & \multirow{2}{*}{$n_2$} & \multirow{2}{*}{$K$} & \multirow{2}{*}{$\Expect \norm{B}_2$} & \multirow{2}{*}{$\textrm{std}\pth{\norm{B}_2}$} & \multicolumn{2}{c}{$q=95\%$} & \multicolumn{2}{c}{$q=99\%$} & \multicolumn{2}{c}{$q=99.9\%$} & \multicolumn{2}{c}{$q=99.99\%$} \\
		\cline{6-7} \cline{8-9} \cline{10-11} \cline{12-13}
		~ & ~ & ~ & ~ & ~ & $c_0$ & $\delta_0$ & $c_0$ & $\delta_0$ & $c_0$ & $\delta_0$ & $c_0$ & $\delta_0$ \\
    \midrule
        32 & 32 & 1 & 1.087 & 0.038 & 1.15 & 0.029 & 1.183 & 0.041 & 1.206 & 0.059 & 1.218 & 0.077 \\ 
        32 & 32 & $\sqrt{3}$ & 1.882 & 0.066 & 1.996 & 0.029 & 2.044 & 0.041 & 2.069 & 0.059 & 2.131 & 0.077 \\ 
        32 & 64 & 1 & 0.941 & 0.027 & 0.988 & 0.014 & 1.015 & 0.021 & 1.039 & 0.03 & 1.042 & 0.039 \\ 
        32 & 64 & $\sqrt{3}$ & 1.631 & 0.046 & 1.707 & 0.014 & 1.743 & 0.021 & 1.786 & 0.03 & 1.797 & 0.039 \\ 
        32 & 128 & 1 & 0.836 & 0.018 & 0.867 & 0.007 & 0.878 & 0.01 & 0.895 & 0.015 & 0.902 & 0.019 \\ 
        32 & 128 & $\sqrt{3}$ & 1.449 & 0.032 & 1.503 & 0.007 & 1.528 & 0.01 & 1.577 & 0.015 & 1.579 & 0.019 \\ 
        32 & 256 & 1 & 0.762 & 0.013 & 0.784 & 0.004 & 0.794 & 0.005 & 0.806 & 0.007 & 0.811 & 0.01 \\ 
        32 & 256 & $\sqrt{3}$ & 1.319 & 0.022 & 1.357 & 0.004 & 1.371 & 0.005 & 1.39 & 0.007 & 1.393 & 0.01 \\ 
        32 & 512 & 1 & 0.708 & 0.009 & 0.723 & 0.002 & 0.731 & 0.003 & 0.74 & 0.004 & 0.747 & 0.005 \\ 
        32 & 512 & $\sqrt{3}$ & 1.226 & 0.016 & 1.253 & 0.002 & 1.267 & 0.003 & 1.278 & 0.004 & 1.283 & 0.005 \\
        64 & 64 & 1 & 1.114 & 0.026 & 1.158 & 0.014 & 1.183 & 0.021 & 1.205 & 0.03 & 1.209 & 0.039 \\ 
        64 & 64 & $\sqrt{3}$ & 1.932 & 0.045 & 2.009 & 0.014 & 2.045 & 0.021 & 2.07 & 0.03 & 2.086 & 0.039 \\ 
        64 & 128 & 1 & 0.959 & 0.018 & 0.992 & 0.007 & 1.005 & 0.01 & 1.04 & 0.015 & 1.054 & 0.019 \\ 
        64 & 128 & $\sqrt{3}$ & 1.66 & 0.031 & 1.711 & 0.007 & 1.743 & 0.01 & 1.782 & 0.015 & 1.785 & 0.019 \\ 
        64 & 256 & 1 & 0.848 & 0.012 & 0.868 & 0.004 & 0.88 & 0.005 & 0.887 & 0.007 & 0.888 & 0.01 \\ 
        64 & 256 & $\sqrt{3}$ & 1.47 & 0.021 & 1.508 & 0.004 & 1.523 & 0.005 & 1.53 & 0.007 & 1.554 & 0.01 \\ 
        64 & 512 & 1 & 0.77 & 0.008 & 0.785 & 0.002 & 0.792 & 0.003 & 0.796 & 0.004 & 0.801 & 0.005 \\ 
        64 & 512 & $\sqrt{3}$ & 1.333 & 0.015 & 1.359 & 0.002 & 1.371 & 0.003 & 1.388 & 0.004 & 1.392 & 0.005 \\ 
        128 & 128 & 1 & 1.131 & 0.017 & 1.159 & 0.007 & 1.173 & 0.01 & 1.199 & 0.015 & 1.205 & 0.019 \\ 
        128 & 128 & $\sqrt{3}$ & 1.956 & 0.029 & 2.008 & 0.007 & 2.024 & 0.01 & 2.044 & 0.015 & 2.045 & 0.019 \\ 
        128 & 256 & 1 & 0.969 & 0.012 & 0.99 & 0.004 & 0.999 & 0.005 & 1.012 & 0.007 & 1.013 & 0.01 \\ 
        128 & 256 & $\sqrt{3}$ & 1.679 & 0.019 & 1.712 & 0.004 & 1.728 & 0.005 & 1.743 & 0.007 & 1.746 & 0.01 \\ 
        128 & 512 & 1 & 0.856 & 0.008 & 0.87 & 0.002 & 0.875 & 0.003 & 0.881 & 0.004 & 0.885 & 0.005 \\ 
        128 & 512 & $\sqrt{3}$ & 1.482 & 0.014 & 1.507 & 0.002 & 1.52 & 0.003 & 1.527 & 0.004 & 1.528 & 0.005 \\ 
        256 & 256 & 1 & 1.14 & 0.011 & 1.16 & 0.004 & 1.17 & 0.005 & 1.18 & 0.007 & 1.181 & 0.01 \\ 
        256 & 256 & $\sqrt{3}$ & 1.976 & 0.021 & 2.01 & 0.004 & 2.027 & 0.005 & 2.036 & 0.007 & 2.036 & 0.01 \\ 
        256 & 512 & 1 & 0.976 & 0.008 & 0.989 & 0.002 & 0.995 & 0.003 & 1.002 & 0.004 & 1.014 & 0.005 \\ 
        256 & 512 & $\sqrt{3}$ & 1.691 & 0.013 & 1.714 & 0.002 & 1.727 & 0.003 & 1.735 & 0.004 & 1.735 & 0.005 \\ 
        512 & 512 & 1 & 1.146 & 0.007 & 1.159 & 0.002 & 1.163 & 0.003 & 1.172 & 0.004 & 1.174 & 0.005 \\ 
        512 & 512 & $\sqrt{3}$ & 1.985 & 0.012 & 2.006 & 0.002 & 2.015 & 0.003 & 2.033 & 0.004 & 2.04 & 0.005 \\ 
        \bottomrule
    \end{tabular}
\end{table}

\subsection{Constant in Lemma \ref{lemma:latala}} \label{appendix:subsec:latala}
Lemma \ref{lemma:latala} shows that there exists a universal constant $C$ such that, for any random matrix $A$ whose entries are independent, we have \begin{equation} \label{eq:latala_numerical}
        \Expect\norm{A}_2 \le C \qth{\max_i \pth{\sum_j \Expect A_{i, j}^2}^{\frac{1}{2}} + \max_j \pth{\sum_i \Expect A_{i, j}^2}^{\frac{1}{2}} + \pth{\sum_{i,j}\Expect A_{i, j}^4}^{\frac{1}{4}}}.
    \end{equation}
    
We use this lemma many times to bound the $L_2$ norm of various random matrices, e.g., in \eqref{eq:apply_latala} and \eqref{eq:filter_norm_CNN}. In the following, we consider the cases where the elements of $A \in \reals^{d \times d}$ follows $U \defeq \calU\qth{-\sqrt{\frac{3}{d}}, \sqrt{\frac{3}{d}}}$ and $\calN\pth{0, \frac{K}{d}}$ for some positive constant $K$, respectively. We also consider the case where we initialize the elements of $A$ by samples of $\calN\pth{0, \frac{1}{d}}$, but we set $\lfloor d^{2-\alpha} \rfloor$ entries to be zero randomly (thus it aligns with the use case in \eqref{eq:apply_latala}).

In the numerical experiments, we generate in total $N=500$ random matrices $A$ and calculate the quantities $\Expect\norm{A}_2, \max_i \pth{\sum_j \Expect A_{i, j}^2}^{\frac{1}{2}}, \max_j \pth{\sum_i \Expect A_{i, j}^2}^{\frac{1}{2}}$ and $\pth{\sum_{i,j}\Expect A_{i, j}^4}^{\frac{1}{4}}$. In Table \ref{table:2}, we report the minimum $C$ such that \eqref{eq:latala_numerical} holds with the choices of $d$, distribution of $A_{i,j}$, and $\alpha$ (if necessary). 

\begin{table}[h!]
\linespread{2} 
    \caption{Numerical results for the constant in Lemma \ref{lemma:latala}}
    \label{table:2}
    \centering
    \begin{tabular}{cccccccc}
    \toprule
		$d$ & Distribution & $\alpha$ & $\max_i \pth{\sum_j \Expect A_{i, j}^2}^{\frac{1}{2}}$ & $\max_j \pth{\sum_i \Expect A_{i, j}^2}^{\frac{1}{2}}$ & $\pth{\sum_{i,j}\Expect A_{i, j}^4}^{\frac{1}{4}}$ & $\Expect\norm{A}_2$ & $C$\\
    \midrule 
        32 & $U$ & N/A & 1.006 & 1.006 & 1.159 & 1.888 & 0.596 \\ [3pt]
        64 & $U$  & N/A & 1.006 & 1.005 & 1.159 & 1.934 & 0.61 \\ [3pt]
        128 & $U$  & N/A & 1.005 & 1.005 & 1.159 & 1.958 & 0.618 \\ [3pt]
        256 & $U$  & N/A & 1.004 & 1.003 & 1.158 & 1.976 & 0.624 \\ [3pt]
        512 & $U$  & N/A & 1.003 & 1.002 & 1.158 & 1.985 & 0.627 \\[3pt]
        32 & $\calN\pth{0, \frac{1}{d}}$ & N/A & 1.011 & 1.014 & 1.314 & 1.905 & 0.571 \\ [3pt]
        64 & $\calN\pth{0, \frac{1}{d}}$ & N/A & 1.008 & 1.008 & 1.315 & 1.947 & 0.585 \\ [3pt]
        128 & $\calN\pth{0, \frac{1}{d}}$ & N/A & 1.005 & 1.007 & 1.316 & 1.965 & 0.59 \\ [3pt]
        256 & $\calN\pth{0, \frac{1}{d}}$ & N/A & 1.005 & 1.006 & 1.316 & 1.979 & 0.595 \\ [3pt]
        512 & $\calN\pth{0, \frac{1}{d}}$ & N/A & 1.004 & 1.004 & 1.316 & 1.988 & 0.598 \\ [3pt]
        32 & $\calN\pth{0, \frac{3}{d}}$ & N/A & 1.751 & 1.745 & 2.279 & 3.295 & 0.571 \\ [3pt]
        64 & $\calN\pth{0, \frac{3}{d}}$ & N/A & 1.755 & 1.744 & 2.28 & 3.361 & 0.582 \\ [3pt]
        128 & $\calN\pth{0, \frac{3}{d}}$ & N/A & 1.742 & 1.743 & 2.28 & 3.405 & 0.591 \\ [3pt]
        256 & $\calN\pth{0, \frac{3}{d}}$ & N/A & 1.743 & 1.742 & 2.28 & 3.428 & 0.595 \\ [3pt]
        512 & $\calN\pth{0, \frac{3}{d}}$ & N/A & 1.74 & 1.739 & 2.279 & 3.441 & 0.598 \\ [3pt]
        32 & $\calN\pth{0, \frac{1}{d}}$ & 0.01 & 0.626 & 0.63 & 1.033 & 1.237 & 0.54 \\ [3pt]
        64 & $\calN\pth{0, \frac{1}{d}}$ & 0.01 & 0.632 & 0.629 & 1.035 & 1.239 & 0.54 \\ [3pt]
        128 & $\calN\pth{0, \frac{1}{d}}$ & 0.01 & 0.63 & 0.63 & 1.037 & 1.242 & 0.541 \\ [3pt]
        256 & $\calN\pth{0, \frac{1}{d}}$ & 0.01 & 0.63 & 0.63 & 1.039 & 1.246 & 0.542 \\ [3pt]
        32 & $\calN\pth{0, \frac{1}{d}}$ & 0.1 & 0.714 & 0.713 & 1.103 & 1.379 & 0.545 \\ [3pt]
        64 & $\calN\pth{0, \frac{1}{d}}$ & 0.1 & 0.729 & 0.729 & 1.117 & 1.426 & 0.554 \\ [3pt]
        128 & $\calN\pth{0, \frac{1}{d}}$ & 0.1 & 0.744 & 0.744 & 1.129 & 1.459 & 0.558 \\ [3pt]
        256 & $\calN\pth{0, \frac{1}{d}}$ & 0.1 & 0.758 & 0.756 & 1.14 & 1.491 & 0.562 \\ [3pt]
        32 & $\calN\pth{0, \frac{1}{d}}$ & 0.5 & 0.928 & 0.925 & 1.258 & 1.759 & 0.565 \\ [3pt]
        64 & $\calN\pth{0, \frac{1}{d}}$ & 0.5 & 0.95 & 0.948 & 1.275 & 1.831 & 0.577 \\ [3pt]
        128 & $\calN\pth{0, \frac{1}{d}}$ & 0.5 & 0.964 & 0.964 & 1.288 & 1.883 & 0.586 \\ [3pt]
        256 & $\calN\pth{0, \frac{1}{d}}$ & 0.5 & 0.975 & 0.974 & 1.296 & 1.92 & 0.592 \\[3pt]
    \bottomrule
    \end{tabular}
\end{table}

\section{Discussion} \label{appendix:sec:disscussion}
In this section, we discuss some assumptions made to simply the presentations. We provide (possible) ways to avoid them but the detailed proofs are omitted.

\subsection{Independency of Weights in the Target Network}\label{appendix:subsec:independent}

The assumption of independent trained weights satisfied to a certain degree.
Many existing works show that the trained weights are not ``far away'' from the initialization and thus certain levels of independency remains among the trained weights. For example, \citet{Bai2020Beyond} show that the trained weights can be approximated by a Taylor expansion around the initialization and the coefficients of the polynomial are relatively small. This also aligns with the observation from the NTK literature \citep{NEURIPS2018_5a4be1fa} that the trained weights are close to initialization. There are no well-accepted metrics to measure how close are the weights to independency, and thus we assume them to be independent. 

There are other ways to relax independency. For random pruning, independency is assumed so that we can apply the Latala's inequality (Lemma \ref{lemma:latala}). There also exist other versions of spectral norm bounds for sub-Gaussian random matrix with non-i.i.d. entries (Chapter 5 of \citet{pastur2011eigenvalue}) and for a matrix with independent rows and columns \citep{vershynin_2012}. 
For magnitude-based pruning, the assumption is used to derive the explicit form of expectation of order statistics. By assuming an equal correlation between weights, we can also give the explicit forms (Chapter 5 of \citet{david2004order}). The general form of order statistics for dependent uniform samples can be achieved approximately in the same way.

\subsection{With-replacement and Without-replacement Sampling for Random Pruning} \label{appendix:subsec:diff}

Under the random pruning scheme, we select $N$ entries uniformly at random from a $d\times d$ weight matrix and set them to zero. The proposed approach in the beginning of Section \ref{sec:FCN} corresponds to ``with-replacement'' sampling since an entry might be selected multiple times. Another ``without-replacement'' sampling approach refers to selecting $N$ non-overlapping entries from the weight matrix. Note that with a positive probability of $\frac{\binom{d^2}{N}}{d^{2N}}$, the entries selected by the ``with-replacement'' approach have no repeated elements and the two approaches align. In this sense, we can derive the results of the ``without-replacement'' approach from the stated results in this work by simply multiplying the corresponding probability that all selected entries are not repeated. 

\subsection{Global and Layer-wise Magnitude-based Pruning} \label{appendix:subsec:magnitude}

In this paper, the magnitude-based pruning is defined layer-wise as we order the weights in each layer based on magnitude separately and prune the smallest ones. There is also another ``global'' version where the weights of the entire network are sorted and the weights with the smallest magnitudes are pruned. Next we show the connection between these two settings and how to extend the proofs to the global setting.

Suppose that we want to prune a total of $N$ weights in a $l$-layer network. If we treat the small weights as balls and layers as bins, then by Lemma \ref{lemma:ballsinbins}, the maximum load in  each bin is bounded by $O(N/l)$ with high probability. In other words, we expect to see that the appearances of pruned weights in all layers are approximately uniform (the numbers can differ by a constant but not orders of magnitude) with high probability. This is also the reason why we rarely see that the small weights appear in the same layer of a trained network in practice. Under this high-probability event, we get back to the layer-wise magnitude-based pruning setting excepts that the number of weights to be pruned in each layer may vary by a constant. In this sense, the original proofs can be easily revised to fit the global magnitude-based setting.

\end{document}